\documentclass[11pt]{article}
\usepackage{amsmath}
\usepackage{graphicx,psfrag,epsf}
\usepackage{enumerate}
\usepackage{natbib}
\usepackage{url} 
\usepackage{amssymb, amsthm, dsfont, mathtools,color}
\usepackage{subcaption,float,mathtools,ytableau}
\usepackage{times}
\usepackage{enumitem}
\usepackage{breakcites}

\usepackage[margin=1.25in]{geometry}

\newtheorem*{theorem*}{Theorem}
\newtheorem{theorem}{Theorem}

\newtheorem{lemma}[theorem]{Lemma}
\newtheorem{corollary}[theorem]{Corollary}

\newtheorem{proposition}[theorem]{Proposition}

\newcommand{\tr}{\operatorname{tr}}
\newcommand{\RR}{\mathbb{R}}

\newcommand{\CC}{\mathbb{C}}
\newcommand{\EE}{\mathbb{E}}
\newcommand{\var}{\operatorname{Var}}

\newcommand{\sign}{\operatorname{sign}}
\newcommand{\ii}{\operatorname{i}}

\newcommand{\indi}{\mathds{1}}

\newcommand{\simg}{\mathbb{S}_d}

\newcommand{\Fcal}{\mathcal{F}}
\newcommand{\Hcal}{\mathcal{H}}

\newcommand{\Pcal}{\mathcal{P}}
\newcommand{\Xcal}{\mathcal{X}}
\newcommand{\Ocal}{\mathcal{O}}

\newcommand{\probP}{P}
\newcommand{\probQ}{Q}
\newcommand{\kernel}{k}
\newcommand{\kendall}{\kernel_\tau}
\newcommand{\band}{\nu}
\newcommand{\mallows}{\kernel_m^\band}
\newcommand{\mmd}{\operatorname{MMD}}
\newcommand{\numObj}{d}

\newcommand{\EmbdPoly}[1]{\Phi_{#1}}
\newcommand{\EmbdNor}[1]{\overline \Phi_{#1}}

\newcommand{\twosets}{\mathcal{T}_{\text{un}}}

\newcommand{\defn}{\ensuremath{: \, = }}
\long\def\comment#1{}
\newcommand{\usedim}{d}
\newcommand{\real}{\RR}

\newcommand{\polyKer}[1]{\kernel^{#1}}
\newcommand{\NorPolyKer}[1]{\overline\kernel^{#1}}
\newcommand{\LamPolyKer}[2]{\overline\kernel^{#1,#2}}
\newcommand{\power}{p}

\begin{document}

\title{\bf On kernel methods for covariates that are rankings }
\author{
Horia Mania$^{\epsilon}$ \\
\and
Aaditya Ramdas$^{\epsilon,\sigma}$\\
\and
Martin J. Wainwright$^{\epsilon, \sigma}$
\thanks{Martin J. Wainwright was partially
supported by Air Force Office of Scientific Research
AFOSR-FA9550-14-1-0016, Office of Naval Research DOD ONR-N00014, and
NSF grant CIF-31712-23800.}\\
\and
 Michael I. Jordan$^{\epsilon, \sigma}$
\thanks{Michael I. Jordan was partially supported by the Mathematical Data Science program of the
Office of Naval Research}\\
 \and
  Benjamin Recht$^{\epsilon}$\thanks{Benjamin Recht was partially supported by
ONR awards N00014-15-1-2620, N00014-13-1-0129, and N00014-14-1-0024
and NSF awards CCF-1148243 and CCF-1217058.}\\
\and
  \hspace{5in}
  \and
{\small \texttt{hmania, aramdas, wainwrig, jordan, brecht ~@berkeley.edu}}\\
Departments of  Statistics$^{\sigma}$ and EECS$^{\epsilon}$, \\ University of California,
 Berkeley, CA, 94720
}
 
\maketitle

\begin{abstract}
  \noindent Permutation-valued features arise in a variety of
  applications, either in a direct way when preferences are elicited
  over a collection of items, or an indirect way in which
  numerical ratings are converted to a ranking. To date, there has been relatively
  limited study of regression, classification, and testing
  problems based on permutation-valued features, as opposed to
  permutation-valued responses.  This paper studies the use of
  reproducing kernel Hilbert space methods for learning 
  from permutation-valued features.  
  These methods embed the rankings into
  an implicitly defined function space, and allow for efficient
  estimation of regression and test functions in this richer space.
  Our first contribution is to characterize both the feature spaces
  and spectral properties associated with two kernels for
  rankings, the Kendall and Mallows kernels.  Using tools from
  representation theory, we explain the limited expressive power of
  the Kendall kernel by characterizing its degenerate spectrum, and in
  sharp contrast, we prove that Mallows' kernel is universal and
  characteristic.  We also introduce families of polynomial kernels
  that interpolate between the Kendall (degree one) and Mallows'
  (infinite degree) kernels.  We show the practical effectiveness of
  our methods via applications to Eurobarometer survey data as well as
  a Movielens ratings dataset.
\end{abstract}

\noindent%

\date{}

\section{Introduction}

Ranking data arises naturally in any context in which preferences are
expressed over a collection of alternatives. Familiar examples include
election data, ratings of consumer items, or choice of schools. Preferences
can be expressed directly via relative comparisons of alternatives, or
indirectly via scores assigned to the different alternatives.  Preferences
are also often expressed implicitly; e.g., through click activity on
the web. In this paper, we consider datasets in which each covariate
corresponds to a complete ranking over a set of $\numObj$ alternatives---that
is, a permutation belonging to the symmetric group. The inferential problems
that we consider are regression, classification and testing problems with rankings
as covariates.

As a running example to which we return in
Section~\ref{sec:survey_data}, consider the Eurobarometer 55.2 survey
conducted in several European countries in 2001, recently published by
the~\cite{eu200155}.  Each respondent was asked to indicate their
preferences over sources of information about scientific developments;
their options were: TV, radio, newspapers/magazines, scientific
magazines, the internet, and school/university.  Therefore, each
observation in the survey contained a ranking of $d=6$ objects, along
with other covariates such as the participant's age, gender, etc; a
snippet is shown in Table~\ref{tab:survey_data}.  Many natural
questions arise from this dataset.  Can we predict a person's
age/gender from their ranking? Do men and women (or old and young)
have the same distribution over sources of information? The primary
goal of this paper is to develop and analyze some principled
methods for answering such questions. 

\begin{table}[h!]
\centering\footnotesize
\begin{tabular}{|c|c|c|c|}\hline
Respondent & Gender & Age & Ranking of news sources \\\hline\hline
 $1$ & F & $32$ & TV $>$ Radio $>$ School/University $>$ Newspapers/Mags. $>$ Web $>$ Sci. Mags.  \\\hline
 $2$ & F & $84$ & TV $>$ Radio $>$ Newspapers/Mags. $>$ School/University $>$ Sci. Mags.  $>$ Web  \\\hline
$3$ & F & $65$ & TV $>$ Newspapers/Mags. $>$ Sci. Mags.  $>$ Radio $>$ School/University $>$ Web  \\\hline
$4$ & M & $29$ & Web $>$ Radio $>$ Newspapers/Mags. $>$ TV $>$ Sci. Mags. $>$ School/University  \\\hline
\end{tabular}
\caption{
Snippet of the Eurobarometer $55.2$ survey data.  
}
\label{tab:survey_data} 
\end{table}

There is a large existing literature on the use of rank statistics
for testing and inference; for instance, see the book by~\cite{lehmann2006nonparametrics}
and references therein.  However, this body of work does not address
problems in which the ranking themselves act as covariates. Thus,
inferential problems in which the rankings are naturally viewed as
covariates are generally simplified in various ways. For example,
in the original report on the Eurobarometer survey data by
the~\cite{eu200155}, the authors measured only the frequency with
which each of the six sources of information was ranked in the first
or second position. Their analysis did not distinguish between
respondents' first and the second preferences and disregarded the
information encoded in their bottom four preferences. When covariates
have been included, the analysis is generally strongly parametric;
for example, \cite{francis2010modeling} analyze the same dataset by
extending the classical Bradley-Terry model~(\citeyear{bradley1952rank}) to
incorporate covariates such as sex and age.

Our focus in the current paper is on nonparametric models in which the
covariates are rankings.  We build on work of \cite{jiao2015kendall}, 
who discuss the use of Mercer kernels for ranking data.  Kernels on
the symmetric group induce an inner-product structure on permutations
by implicitly embedding them into a suitable Reproducing Kernel Hilbert
Space (RKHS).  This space is defined by a bivariate kernel function, and the representer
theorem~\citep{KimWah71} allows problems of regression and testing to
be reduced to the computation of the kernel values $k(\sigma,\pi)$ for
pairs of permutations $(\sigma, \pi)$.  We view kernel-based
methodology as particularly appropriate for ranking problems: in
particular, it allows us to transition from the cumbersome setting of
the non-Abelian symmetric group of permutations to the familiar
setting of Hilbert spaces.  This methodology does not require
us to make generative or probabilistic assumptions, and is practically
viable as long as kernel evaluations are computationally efficient.


\subsection{Kernels on the symmetric group}
  
There is a rich theoretical understanding of kernels on Euclidean
spaces, such as the linear, polynomial, Matern, Laplace and Gaussian
kernels~\citep{learningkernels}. The latter two are especially
popular because they are \emph{translation-invariant}, meaning that
$k(x,y) = k(x+z,y+z)$ for all $x,y,z\in \RR^\numObj$,
\emph{characteristic}, meaning the maximum mean discrepancy over the
unit ball of the RKHS defines a metric on the space of distributions,
and \emph{universal}, meaning the RKHS is dense in the space of
square-integrable functions.  The latter property ensures that any
square-integrable decision boundary or regression function can be
approximated arbitrarily well by a sequence of elements from the RKHS.

Permutations lie within the symmetric group, and various kernels for
this non-Abelian group have been proposed
\citep{kondor2002diffusion,kondor2010ranking, jiao2015kendall}.  Many
of these kernels, including the Kendall and Mallows' kernels considered
in this paper, are \emph{right-invariant}, meaning that they are
invariant to a re-indexing of the underlying objects.  This property
is desirable for our applications: otherwise, the kernel similarity
between a pair of permutations would depend on how the items were
indexed. Much of the focus in the theoretical literature on kernel
methods has focused on the \emph{bi-invariant} class of kernels;
these are kernels that are both right- and left-invariant. A prominent
example of a bi-invariant kernel is the diffusion kernel,
which are quite well understood \citep[see, e.g.,][]{kondor2008group}.
Unfortunately, such kernels are not suitable for our applications, since for any
bi-invariant kernel, the value between a pair of rankings that rank
a specific item in positions one and two respectively would be the
same as if they ranked it in positions, say, one and twenty.  We thus
focus on right-invariant kernels, such as Kendall's and Mallows' kernels
on the symmetric group, and aim to bring the understanding of these 
kernels to the level of the bi-variant kernels.  In particular we
analyze the feature maps and spectral properties of the Kendall's
and Mallows'kernels, as well as a new class of polynomial kernels.

There is also a mathematical literature on metrics on the symmetric
group (e.g., Cayley's metric, Ulam's metric, and Spearman's footrule).
However, with the important exception of nearest-neighbor methods,
most statistical analysis methods are more compatible with inner-product
representations (kernels, similarities) than with metrics (distances,
dissimilarities).

\subsection{Contributions}

After a presentation of basic background on kernel methods on the
symmetric group in Section~\ref{SecKernelMethod}, we begin our 
development by presenting an analysis of Kendall's and Mallows'
kernels from a primal point of view in Section~\ref{SecFeatures}.
In particular, in Proposition~\ref{PropKendallMap}, we prove that the Kendall
kernel Gram matrix always has rank $\binom{\numObj}{2}$, and discuss
the statistical implications of this result.  Then, in
Proposition~\ref{PropMallowsMap}, we present a novel
finite-dimensional feature map for the Mallows' kernel.
This result is surprising, because the Mallows' kernel is 
the analog for permutations of the Gaussian kernel.  The
latter does \emph{not} have a finite-dimensional feature map.

In Euclidean spaces, there exists a large body of work on the spectral
properties of kernels (the decay of the eigenvalues of their Fourier transforms).
This informs the statistical analysis of kernel methods, providing leverage
on the ability of kernels to discriminate between distributions, or estimate
decision boundaries and regression functions.  Motivated by this, in
Section~\ref{SecKendallMallowsFreq} we study the spectra of Kendall's
and Mallows' kernels, proceeding via a non-Abelian variant of
Bochner's theorem. This analysis requires a foray into representation
theory~\citep{fulton1991representation}. We provide as much background
on representation theory as is necessary to understand our theorem
statements, leaving the proofs and fuller development of representation
theory for the supplementary material. Theorem~\ref{ThmKendallTransform}
fully characterizes the Fourier spectrum of the Kendall kernel. In particular we show that it
has only two nonzero irreducible representations, both of which turn
out to be rank-one matrices; this degeneracy suggests its strength as
a kernel is useful only in a limited range of
problems. Theorem~\ref{ThmMallows} provides a first-principles proof
of the fact that the Mallow's kernel is universal and characteristic;
i.e, every irreducible representation is a strictly positive-definite
matrix.

In Section~\ref{SecPolyKernels}, we propose and analyze natural
families of polynomial kernels of degree $p$ that interpolate between
the Kendall and Mallows kernels (corresponding to $p=1$ and $p =
\infty$ respectively). We study their (primal) feature maps and (dual)
spectra and in Theorem~\ref{ThmPoly}, we prove that $p=d-1$ suffices
for the kernel to be universal and characteristic.

In addition to these theoretical insights, we also present the results
of various experiments with our kernel representations.  In our first
set of experiments, we apply kernel methods to a simulated data set in
order to illustrate our predicted differences in the empirical power
of two-sample hypothesis tests using different kernels for rankings,
and discuss on which instances we expect Kendall's or Mallows' kernels
to have higher power. We then apply these kernel-based tests to the
Eurobarometer survey data, and also fitted kernel SVM and kernel
regression models to this data in order to showcase the usefulness of
kernel methods to leverage ranking data. Our two-sample tests find
that men and women do have significantly different preferences, the
classifiers have a test error of 34\% for predicting if the respondent
was old or young; and the regression from rankings to age has a test
prediction error of about 11 years.  Moreover, we studied a data set
consisting of ratings for movies, in which we transformed the users'
ratings across movie genres into rankings. We find signficant
evidence for males and females having different preferences over movie
genres, a simple illustration of the possible utility in converting
absolute ratings into relative rankings.

\comment{
\paragraph{Paper organization.} To summarize, the remainder of this paper is structured as
follows. Section~\ref{SecKernelMethod} provides background on kernel
methods in the context of rankings, and further discusses our
goals. Section~\ref{SecFeatures} presents our results regarding the
feature maps of Kendall's and Mallows' kernels.  In
Section~\ref{SecKendallMallowsFreq} we introduce some basics of
representation theory for the symmetric group, and we present and
discuss the two theorems characterizing the spectra of Kendall's and
Mallows' kernels. Section~\ref{SecPolyKernels} analyzes natural
families of polynomial kernels. In Section~\ref{sec:experiments} we
present our empirical findings.  We conclude in Section~\ref{SecConc}
with a summary of our results and some interesting open questions. The
proofs of our results are included in the supplementary material.
}


\section{An overview of kernel methods for rankings}
\label{SecKernelMethod}

In order to understand the use of kernels for permutation-valued
features, we first need to introduce some standard terminology.

\paragraph{Symmetric group $\simg$.} There is a natural one-to-one correspondence 
between permutations and rankings. Indeed, the set 
$[d] \defn \{1, 2, \ldots, d\}$ can represent both the labels of a collection 
of $d$ objects and the rankings of these items. 
For any permutation \mbox{$\sigma : [d] \to [d]$} 
we can view $\sigma(i)$ as the rank of
object $i$. The set of all permutations forms a group with the
standard function composition $\sigma \circ \sigma'$; that is,
we have $\pi = \sigma \circ \sigma' \iff \pi(i) = \sigma(\sigma'(i))$.  This group is
known as the symmetric group on $d$ elements and it is denoted by
$\simg$.

\paragraph{Universal RKHS.} A kernel is a bivariate function, $\kernel\colon \simg \times \simg \rightarrow \RR$, 
such that for any collection of rankings, the associated Gram matrix is positive semi-definite.  We let $\Fcal_k$ denote the reproducing kernel Hilbert space (RKHS) induced by the kernel $k$, which is a set of functions defined by the closure of the span of $\{k(\sigma,\cdot)\}_{\sigma \in \simg}$. We also define the RKHS inner product between two functions $f=\sum_{j=1}^{d!} a_j k(\sigma_j,\cdot)$ and $g=\sum_{j=1}^{d!} b_j k(\sigma_j,\cdot)$ to be $\langle f,g \rangle_\Fcal = \sum_k \sum_j a_k b_j k(\sigma_k,\sigma_j)$.  This inner product induces the RKHS norm $\|f\|_{\Fcal_k} = \sqrt{\langle f, f \rangle_{\Fcal_k}}$. If $k$ is a kernel on a space $\mathcal{X}$ (say $\simg$) and $\ell$ is a kernel on $\mathcal{Y}$ (say $\RR^p$), then $m := k\times \ell$ is a kernel on the space $\mathcal{Z} = \mathcal{X}\times\mathcal{Y}$; that is, for $z=(x,y),z'=(x',y')$, we have $m(z,z')=k(x,x')\ell(y,y')$. Naturally, we can recurse this process to define kernels on domains involving a variety of data types, showcasing their generality.
For compact metric spaces, 
a continuous kernel $\kernel$ is called \emph{universal} if the RKHS $\Fcal_k$ defined by it is
dense, in $L_\infty$ norm, in the space of continuous functions~\citep{steinwart2002influence}.
In our setting, a kernel $k$ is universal if and only if any function
$f:\simg \to \RR$ can be written as a linear combination of functions
$\kernel(\pi, \cdot)$, with $\pi \in \simg$, that is $\Fcal_k$ contains all possible functions.

\paragraph{Feature maps.} Mercer's theorem \citep[Proposition 2.11]{learningkernels} guarantees that for any kernel $k$ there exists a
Hilbert space $\Hcal$ and a \emph{feature map} $\Phi \colon \simg
\rightarrow \Hcal$ such that
\begin{align*}
\kernel(\sigma, \sigma') = \langle \Phi(\sigma), \Phi(\sigma')
\rangle_{\Hcal} ,\quad \forall \sigma, \sigma' \in \simg.
\end{align*}
Feature maps are not unique, and many different feature maps may give
rise to the same kernel.  The reproducing property states that
$\langle k(\sigma,\cdot), k(\sigma',\cdot) \rangle_{\Fcal_k} =
k(\sigma,\sigma')$. Hence, $k(\sigma,\cdot)$ is one example of a
feature map.  Kernels correspond to inner products in
appropriate feature spaces, and can be thought of as a measure of
similarity between rankings. The kernels that we consider
have feature maps that embed $\simg$ into $\RR^m$, for some finite
dimension $m$, in which case the inner product $\langle \cdot, \cdot
\rangle_{\Hcal}$ represents the standard $m$-dimensional Euclidean
inner product, and the induced norm $\|\cdot\|_\Hcal :=
\sqrt{\cdot,\cdot}$ is the $m$-dimensional Euclidean norm.

\paragraph{Right-invariance.} A bivariate function $F\colon \simg
\times \simg \rightarrow \RR$ is called \emph{right-invariant} if
$F(\sigma, \sigma') = F(\sigma\circ \pi, \sigma'\circ \pi)$ for all
permutations $\sigma, \sigma', \pi \in \simg$. By setting $\pi =
\sigma^{-1}$, we see that this property holds if and only if
$F(\sigma, \sigma') = f(\sigma'\circ \sigma^{-1})$ for some function
\mbox{$f: \simg \rightarrow \RR$.} For kernels, we overload notation
by using $\kernel$ to refer to both $F$ and $f$ by $\kernel$ (usage will be clear
from the context).  Right-invariance of kernels is desirable for
applications involving rankings since it ensures that the kernel
values remain unchanged by a relabeling of the objects being
ranked. Furthermore, as we discuss later, right-invariance enables us
to use Fourier analysis to study the kernels.

\paragraph{Kendall's and Mallows' kernels.} All the kernels that
we study in this paper measure the similarity between two rankings
through the number of pairs of objects that they order in the same way
or in opposite ways. More precisely, letting $n_d(\sigma, \sigma')$
and $n_c(\sigma, \sigma')$ denote (respectively) the number of
\emph{discordant} and \emph{concordant} pairs between permutations
$\sigma$ and $\sigma'$, we have the relations
\begin{subequations}
\begin{align}
  n_d(\sigma,\sigma') & \defn \sum_{i< j}\left[\indi_{\{\sigma(i)<
      \sigma(j)\}}\indi_{\{\sigma'(i)>\sigma'(j)\}} +
    \indi_{\{\sigma(i)>
      \sigma(j)\}}\indi_{\{\sigma'(i)<\sigma'(j)\}}\right], \quad
  \mbox{and} \\
\label{EqnGobiManchurian}
n_c(\sigma, \sigma') & = {d \choose 2} - n_d(\sigma, \sigma'),
\end{align}
\end{subequations}
where equality~\eqref{EqnGobiManchurian} follows because any pair of
indices is either concordant or discordant.  Of particular interest
are \emph{Kendall's kernel} denoted by $k_\tau$, and \emph{Mallows'
  kernel} denoted by $k_m^\band$, where $\band$ is a user-chosen
bandwidth parameter.  They each depend only on the number of
discordant/concordant pairs, and are defined by
\begin{subequations}
\begin{align}
\label{EqnDefnKendall}
\kendall (\sigma, \sigma') & \defn \frac{n_c(\sigma, \sigma') -
  n_d(\sigma, \sigma')}{{d\choose 2}}, \quad \mbox{and} \\
\label{EqnDefnMallows}
\mallows (\sigma, \sigma') &\defn \exp\left({-\band n_d(\sigma, \sigma')}\right).
\end{align}
\end{subequations}

\cite{jiao2015kendall} show that $\kendall$ and $\mallows$ are indeed
kernels and that they can be computed in $\Ocal(\numObj \log \numObj)$
time.  It is not hard to check that the number of discordant pairs
between two permutations is right-invariant, and in fact $n_d(\sigma,
\sigma') = \ii(\sigma' \circ \sigma^{-1})$, where $\ii(\pi)$ denotes
the number of inversions of the permutation $\pi$ (see the
supplementary material for a short proof of this fact).

Therefore, kernels that depend only on the number of discordant or
concordant pairs are right-invariant, which is one of the reasons
behind our particular interest in Kendall's and Mallows' kernels.
Another reason is that the Kendall kernel corresponds almost directly
to the classical Kendall-$\tau$ metric on $\simg$, and the Mallows'
kernel is reminiscent of the popular Mallows' distribution over
$\simg$. Later, we will introduce a family of polynomial kernels that
interpolate between these two kernels.  While these are not the only
kernels of interest, they are natural starting points.

\paragraph{Kernel regression on $\simg$.}

Consider the problem of \emph{kernel ridge regression}, where we
fit a nonlinear model over $\simg$.  This implicitly corresponds to
fitting a linear model in the feature space $\Hcal$. Given a set of $n$ observations
$\{(\pi_i,y_i)\}_{i=1}^n$, kernel ridge regression fits a
function $f\colon \simg \rightarrow \RR$ to the data by solving the
optimization problem
\begin{align}
\label{eqn:opt1}
f^* ~:=~ \arg\min_{f \in \Fcal_k} \sum_{i=1}^n (y_i - f(\pi_i))^2 +
\lambda \|f\|^2_{\Fcal_k},
\end{align}
where $\lambda \in \RR^+$ is a regularization parameter.

 If $k$ is universal, then the estimate $f^*$ can
 approximate any function $f:\simg \to \RR$ arbitrarily
 well. Conversely, if $\Fcal_k$ is not universal, then we may suffer
 from an approximation error even in the limit of infinite data.

 \paragraph{Representer theorem.} Note that kernel ridge regression is
 never directly performed as written above---indeed, the representer
 theorem~\citep{KimWah71} implies that $f^*$ lies in the span of
 $\{k(\pi_i,\cdot)\}_{i=1}^n$, meaning that the optimum $f^* =
 \sum_{i=1}^n w^*_i k(\sigma_i,\cdot)$ for some vector $w^*$. This
 allows us to rewrite the optimization problem~\eqref{eqn:opt1} as:
\begin{align*}
w^* := \arg\min_{w \in \RR^n} \|y - M_k w\|_2^2 + \lambda \|w\|_2^2,
\end{align*}
where $M_k$ is the $n \times n$ Gram matrix whose entries are
$M_{k,ij} = k(\nu_i,\nu_j)$. This in turn is solved by matrix
inversion: $w^* = (M_k + \lambda I_n)^{-1}y$.

\paragraph{Characteristic kernels.} Any kernel on a domain $\Xcal$ induces a semi-metric on the set of
probability distributions on $\Xcal$, known as the \emph{maximum
  mean discrepancy}~\citep[][]{Mul97,Rac13,gretton2012kernel}, which in our setting of 
$\Xcal = \simg$ is given by
\begin{align}
\label{EqMMD}
\mmd_{\kernel} (\probP, \probQ) = \sup_{ \|f\|_{\Fcal_k} \leq 1 }\EE_{\sigma
  \sim P} [f(\sigma)] - \EE_{\pi \sim Q}[f(\pi)].
\end{align}
One can define the mean embedding of $\probP$ using a feature map $\Phi\colon \simg \rightarrow \Hcal$ of the kernel $\kernel$, as $\mu_{\kernel, \probP}= \EE_{\sigma \sim \probP} [\Phi(\sigma)]$. 
Elementary computations \citep[see][]{gretton2012kernel} show that 
\begin{equation}
\label{EqMMDmeans}
\mmd_{\kernel} (\probP, \probQ) = \|\mu_{\kernel, \probP} - \mu_{\kernel, \probQ}\|_\Hcal.
\end{equation}
The kernel is said
to be \emph{characteristic} if $\mmd_\kernel$ actually defines a
metric on the set of probability distributions---that is, if
$\mmd_\kernel(\probP, \probQ) = 0$ if and only if $\probP = \probQ$.

\paragraph{Two-sample testing on $\simg$.}

Let $\probP$ and $\probQ$ be probability distributions over
$\simg$, and consider testing the null $H_0 : P = Q$ against the alternative $H_1 : P \neq Q$, using samples $\alpha_1$, $\alpha_2$, \ldots, $a_{n_1} \stackrel{i.i.d.}{\sim} \probP$ and  $\beta_1$, $\beta_2$, \ldots, $\beta_{n_2} \stackrel{i.i.d.}{\sim}  \probQ$.
One approach to this testing problem is to estimate a (semi-)metric between $P$ and $Q$, and reject $H_0$ if the estimate is large.
For example, \cite{gretton2012kernel} define the statistic
\begin{equation}
\label{eq:u-statistic}
T_\kernel(\alpha,\beta) = \frac1{n_1(n_1 -1)}\sum_{i\neq j} \kernel(\alpha_i,\alpha_j) + \frac1{n_2(n_2-1)}\sum_{i\neq j}\kernel(\beta_i,\beta_j) - \frac2{n^2}\sum_{i=1}^{n_1}\sum_{j=1}^{n_2}\kernel(\alpha_i, \beta_j),
\end{equation}
which is an unbiased estimator of $\mmd_\kernel^2$, and define the associated test $T_\kernel(\alpha, \beta) > t^*$ for some threshold $t^*$ (which can be determined, for example, by bootstrap or permutation testing). We will use this nonparametric framework for two-sample testing as a jumping-off point in our investigation of the statistical properties of kernels on permutations.
Specifically, we will investigate the interpretability of this class of tests. Given a kernel $k$, what kind of differences between $\probP$ and $\probQ$ is the test sensitive to? If the null hypothesis is \emph{not} rejected, does that mean the two probability distributions are equal or that they simply have the same low-order moments (for some appropriate notion of moment)? 

One may understand the  test and the $\mmd_\kernel$ semi-metric by studying the kernel $\kernel$.
For example, since $\mmd_\kernel$ is not always a metric, 
this test would have trivial power against alternatives  $\probP \neq \probQ$ whenever $\mmd_\kernel(\probP, \probQ) = 0$. Hence, it is useful to understand when 
the $\mmd_\kernel$ could be equal to zero, even though $\probP \neq \probQ$. 
The results presented in the next section offer answers to these questions. 
For example, Proposition~\ref{PropKendallMap} shows that the
MMD induced by the Kendall kernel is \emph{not} a metric, and in fact it is far from being a metric. 
In sharp contrast, Theorem~\ref{ThmMallows} guarantees that MMD induced by the Mallows kernel is a metric; i.e.,
$\mmd_{\mallows}(\probP, \probQ) = 0$ only when $\probP = \probQ$.

\section{Feature spaces of the Kendall and Mallows kernels}
\label{SecFeatures}

\cite{jiao2015kendall} constructed a feature map  $\Phi_\tau\colon
\simg \rightarrow \RR^{{\numObj \choose 2}}$ for the Kendall kernel defined by:
\begin{align}
\Phi_\tau(\sigma)_{\{i,j\}} = \sqrt{{d\choose
    2}^{-1}}\left(2\indi(\sigma(i) < \sigma(j)) - 1\right) \qquad \mbox{for each $i < j$,}
\end{align}
which is easily seen to satisfy $\kendall(\sigma, \sigma') = \Phi_\tau(\sigma)^\top \Phi_\tau(\sigma')$. 
Using this map we can give an interpretation of the MMD operator of Eq.~\eqref{EqMMDmeans}.
Fix an ordering $\sigma_1$, $\sigma_2$, \ldots, $\sigma_{\numObj!}$ of the elements of $\simg$ and fix an ordering $t_1$, $t_2$, \ldots, $t_{{\numObj \choose 2}}$ of the tuples $(a,b)$ with $a < b$ and $a,b\in [\numObj]$. Denote the set of tuples by $\twosets^*$. 
We define $M_\tau$ to be the $\RR^{{\numObj \choose 2} \times \numObj!}$ matrix whose columns are indexed by the rankings $\sigma_j$, whose rows are indexed by the tuples $t_i$, and whose  $j$-th column is the vector $\Phi_\tau(\sigma_j)$. With this notation, if we view $\probP$ and $\probQ$ as vectors in $[0,1]^{\numObj!}$, the MMD in \eqref{EqMMDmeans} is equal to $\|M_\tau (P - Q)\|_2$. 

\noindent We also define the matrix $A_\tau \in \{0,1\}^{{\numObj \choose 2} \times \numObj!}$ with columns and rows indexed similarly, and entries
\begin{align*}
\left(A_\tau\right)_{\{a,b\}, \sigma} = \left\{
\begin{array}{cc}
1 & \text{if } \sigma(a) < \sigma(b)\\
0 & \text{if } \sigma(a) > \sigma(b).
\end{array}\right.
\end{align*} 
We then have the following proposition. 

\begin{proposition}
\label{PropKendallMap}
The maximum mean discrepancy $\mmd_{\kendall}$ between two probability
distributions $\probP$, $\probQ$ on $\simg$ is zero if and only if
$A_\tau(\probP - \probQ) = 0$. Moreover, the matrix $A_\tau$ has rank
${\numObj \choose 2}$.
\end{proposition}

Straightforward algebraic manipulations show that
$\frac{1}{2}\sqrt{{\numObj \choose 2}}M_\tau(P - Q) = A_\tau(P - Q)$,
proving the first part of the above proposition, and the second part
is proved in the supplement.  We remark that $(A_\tau P)_{\{a,b\}} =
\probP(\sigma(a) < \sigma(b))$, and hence the $\mmd_{\kendall}$
corresponds to the Euclidean distance between the vectors of
probabilities of the events $\{\sigma: \sigma (a) < \sigma(b)\}$ under
the distributions $\probP$ and $\probQ$. As a parallel to the linear
kernel in $\RR^m$, the Kendall kernel detects a difference between two
probability distributions only if they differ in mean, where we define
the mean as the vectors of probabilities of events $\{\sigma \colon
\sigma(a) < \sigma(b)\}$.

How many probability distributions have the same mean embedding under
the Kendall kernel as $\probP$? A probability $\probQ$ over $\simg$ is
a vector in $\RR^{\numObj!}$ that is contained in the unit simplex, a
subset of a hyperplane of dimension $\numObj! -
1$. Proposition~\ref{PropKendallMap} shows that for each $\probP$ in
the interior of the unit simplex of $\RR^{\numObj!}$ there is a
subspace $V \in \RR^{\numObj!}$ of dimension $\numObj! - {\numObj
  \choose 2} - 1$ such that for each $\gamma \in V$ there exists
$\epsilon > 0$ such that $\probP + \epsilon \gamma$ is a probability
distribution over $\simg$ and $\mmd_{\kendall}(\probP +
\epsilon\gamma, \probP) = 0$. In other words, as $\numObj$ increases,
the fraction of the directions that the Kendall kernel cannot
distinguish goes to one. This observation shows that the Kendall
kernel is far from being a metric on the probability simplex in
$\RR^{\numObj!}$. We offer a Fourier transform perspective on this
fact in Theorem~\ref{ThmKendallTransform}; showing in particular that the Kendall
kernel can detect only low-frequency differences between two
probability distributions.

We next describe a finite-dimensional feature map for the Mallows
kernel.

\begin{proposition}
\label{PropMallowsMap}
The feature map of the Mallows kernel $\mallows$ is given by a map
$\Phi_m \colon \simg \to \Pcal(\twosets^*)$, where $\Pcal(\twosets^*)$
denotes the power set of $\twosets^*$. If $s_1, s_2, \ldots, s_r$ are
distinct elements of $\twosets^*$, we have
\begin{equation} 
\label{eq:feature_map_mallows}
\Phi_m(\sigma)_{s_1s_2\ldots s_r} = \left(\frac{1 +
  \exp({-\band})}{2}\right)^{\frac{1}{2}{\numObj \choose 2}}
\left(\frac{1 - \exp({-\band})}{1 +
  \exp({-\band})}\right)^{\frac{r}{2}}\prod_{i = 1}^r
\EmbdNor{}(\sigma)_{s_i},
\end{equation}
where $\EmbdNor{}(\sigma)_{s_i} = 2\indi_{\{\sigma(a_i) <
  \sigma(b_i)\}} - 1$ when $s_i = (a_i, b_i)$, and
$\psi(\sigma)_\emptyset = 2^{-\frac{1}{2}{\numObj \choose 2}}(1 +
\exp({-\band}))^{\frac{1}{2}{\numObj \choose 2}}$ for all $\sigma \in
\simg$.
\end{proposition}

While mean embeddings with respect to the Kendall kernel correspond to
the probabilities of the events $\{\sigma: \sigma (a) < \sigma(b)\}$,
the mean embeddings with respect to the Mallows kernel correspond to
the probabilities of the events defined by prescribing all orderings
on subsets $\{a_1, a_2, \ldots, a_k\}$ of objects of $[d]$. In
comparison to the Kendall kernel, it is apparent from
Eq.~\eqref{eq:feature_map_mallows} that the Mallows kernel captures
more features of probability distributions, and hence it can
distinguish more pairs of distributions than the Kendall kernel. In
fact, Theorem~\ref{ThmMallows} to be stated in the sequel shows
that the Mallows kernel is characteristic.


\section{Fourier analysis of the Kendall and Mallows kernels}
\label{SecKendallMallowsFreq}

We start by setting out the basic definitions and concepts that will
allow us to state our results concerning the Fourier transforms of
the Kendall and Mallows kernels; elementary treatments of these
concepts are provided by~\cite{kondor2008group}
and~\cite{huang2009fourier}, with a concise summary given
by~\cite{kondor2010ranking}.  Our proofs actually require more
extensive machinery from the theory of Fourier analysis on groups, as
found, for example,
in~\cite{diaconis1988group},~\cite{sagan2013symmetric},
or~\cite{fulton1991representation}.  We introduce these more advanced
concepts in the supplementary material, which also contains full proofs
of our results.

The Fourier transform of a function $f\colon \simg \rightarrow \CC$
takes the form
\begin{align}
\label{EqFourierTransform}
\widehat f(\rho_{\lambda}) \defn \sum_{\sigma \in \simg} f(\sigma)
\rho_{\lambda}(\sigma),
\end{align}
where $\rho_\lambda$ is a matrix-valued function defined shortly. 
As a contrast with the Fourier transform for functions defined over $\RR$,
instead of being indexed by a frequency $\xi$, the Fourier
transform is indexed by $\lambda$, which is a \emph{partition} of
$d$---a non-increasing sequence of integers that sum to
$\numObj$. Furthermore, instead of the standard exponential basis functions $\exp({i \xi x})$,
the terms $\rho_\lambda$ are functions from $\simg$ to $\CC^{d_\lambda \times d_\lambda}$. 

Let us make these notions more precise. A \emph{representation} of
the symmetric group is a matrix-valued function $\rho\colon \simg
\rightarrow \CC^{d_\rho \times d_\rho}$ such that $\rho(\sigma)$ is
invertible and $\rho(\sigma \circ \sigma') = \rho(\sigma)\rho(\sigma')$ for
all permutations $\sigma, \sigma' \in \simg$. The integer $d_\rho$ is
called the dimension of the representation. As an immediate
consequence of the definitions, it follows that
\begin{align*}
\rho(e) = I_{d_\rho} \quad \mbox{and} \quad \rho(\sigma)^{-1} =
\rho(\sigma^{-1}) \qquad \mbox{for all $\sigma \in \simg$.}
\end{align*}
A representation $\rho$ is \emph{reducible} if it is equivalent to the
direct sum of two representations. To be more explicit, a
representation $\rho$ is reducible if there exist two representations
$\rho_1$ and $\rho_2$ and an invertible matrix $C \in \CC^{d_\rho
  \times d_\rho}$ such that
\begin{align*}
\rho(\sigma) = C^{-1} \left[\rho_1(\sigma)\oplus \rho_2(\sigma)\right]
C = C^{-1} \left(\begin{array}{cc} \rho_1(\sigma) &
  \mathbf{0}\\ \mathbf{0} & \rho_2(\sigma)
\end{array}\right)
C \quad \text{ for all } \sigma \in \simg.
\end{align*} 
A representation that is not reducible is called
\emph{irreducible}. For brevity, we refer to irreducible
representations as \emph{irreps}.  The symmetric group has a finite
number of distinct irreps (an explanation of the meaning of
``distinct'' is provided in the supplementary material), and these
irreps have a standard indexing by finite sequences of positive
integers $\lambda = (\lambda_1, \lambda_2, \ldots, \lambda_r)$ such
that $\lambda_1 \geq \lambda_2 \geq \ldots \geq \lambda_r$ and $\sum_{i =
  1}^r \lambda_i = \numObj$. Such sequences are called
\emph{partitions} of $\numObj$ and $\lambda \vdash \numObj$ means that
$\lambda$ is a partition of $\numObj$.

Returning to equation~\eqref{EqFourierTransform} and using the
terminology just introduced, the Fourier transform of a function on
the symmetric group can be described as a mapping from the irreps
$\rho_\lambda$ to matrices in $\CC^{d_\lambda \times d_\lambda}$. This
version of the Fourier transform shares many similar properties with
its counterpart over real numbers, including the Fourier inversion
formula and the Plancherel formula.  For
future reference, we note that Bochner's theorem in this context states that a 
a right-invariant kernel $k \colon \simg \times \simg \rightarrow \CC$
is positive definite if and only if the matrix $\widehat k(\rho_\lambda)$
is positive semi-definite for all partitions $\lambda \vdash \numObj$.
For more on these properties and other results needed in this paper, 
we refer the reader to the supplementary material.

Before turning to our main results, it is convenient to introduce some
notation for the standard partial ordering of the partitions of $\numObj$. Given
any two partitions $\lambda = (\lambda_1, \lambda_2, \ldots,
\lambda_r)$ and $\mu = (\mu_1, \mu_2,\ldots, \mu_l)$, we say that
$\lambda \unrhd \mu$ if $\sum_{i = 1}^{j} \lambda_i \geq \sum_{i =
  1}^j \mu_i$ for all $1 \leq j \leq \min\{l,r\}$. We say $\lambda
\lhd \mu$ whenever it is not true that $\lambda \unrhd \mu$. The
irreps of the symmetric group inherit the same partial ordering.

Equipped with this background, we now turn to the statements of our
results on the spectral properties of the Kendall and Mallows kernels, 
as well as a discussion of some of their consequences.
We begin with a theorem that characterizes the spectrum of the
Kendall kernel.

\begin{theorem}
\label{ThmKendallTransform}
The Kendall kernel has the following properties:
\begin{enumerate}[leftmargin=*]
\item[(a)] When $d = 2$, the Fourier transform of the Kendall kernel
  is equal to $0$ at $\rho_{(2)}$ and equal to $2$ at $\rho_{(1,1)}$.
\item[(b)] When $\numObj \geq 3$, the Fourier transform
  $\widehat\kendall$ of the Kendall kernel is zero at all irreducible
  representations except for $\rho_{(d-1,1)}$ and
  $\rho_{(d-2,1,1)}$. Furthermore, at both of the latter two
  representations, the Fourier transform $\widehat \kendall$ has rank
  one.
\end{enumerate}
\end{theorem}

Since the Fourier spectrum of a kernel determines its ``richness,'' Theorem~\ref{ThmKendallTransform} offers an alternative perspective 
to Proposition~\ref{PropKendallMap}. The following corollary gives a 
characterization of the discriminative properties of the Kendall kernel in the frequency domain. 
 
\begin{corollary}
When $d \geq 3$, for the Kendall kernel, the MMD semi-metric is given by
\begin{align}
\label{EqnKendallMMD}
\mmd_\tau(P, Q)^2 & = \frac{1}{d!} \sum_{\lambda \in\left\{\substack{(d-1,1),\\
  (d-2,1,1)}\right\}}  d_\lambda \tr\left[\left(\widehat \probP(\rho_\lambda) -
  \widehat \probQ(\rho_\lambda)\right)^\top \widehat{\kendall}(\rho_\lambda) 
\left(\widehat \probP(\rho_\lambda) - \widehat
  \probQ(\rho_\lambda)\right) \right]. 
\end{align}
\end{corollary}

This result follows by combining the Fourier-analytic characterization
of Theorem~\ref{ThmKendallTransform} with a more general expression of
$\mmd_\kernel^2$ in the frequency domain, as presented in the supplementary
material.  Corollary~\ref{EqnKendallMMD} shows
that most differences between $\widehat\probP$ and $\widehat\probQ$ do
not contribute to $\mmd_\kernel(\probP, \probQ)$. The only
differences that contribute to $\mmd_\kernel$ are the $(\numObj -
1) \times (\numObj - 1)$ matrix $\widehat \probP (\rho_{(\numObj -1 ,1)}) -
\widehat \probQ (\rho_{(\numObj -1 ,1)})$ and the ${\numObj - 1
  \choose 2} \times {\numObj - 1 \choose 2}$ matrix $\widehat \probP
(\rho_{(\numObj - 2 , 1, 1)}) - \widehat \probQ (\rho_{(\numObj -2 ,1,
  1)})$.  To be more precise, the Kendall kernel can differentiate
between $\probP$ and $\probQ$ if and only if their Fourier transforms
at $\rho_{(d-1,1)}$ or $\rho_{(d-2,1,1)}$ differ along a single
direction aligning with the only eigenvector with a non-zero
eigenvalue of $\widehat\kendall(\rho_{(d-1,1)})$ or
$\widehat\kendall(\rho_{(d-2,1,1)})$.

We now turn to Fourier analysis of the Mallows
kernel~\eqref{EqnDefnMallows}.  Despite its superficial similarity
to the Kendall kernel, it has very different properties. 
\begin{theorem}
\label{ThmMallows}
The Fourier transform $\widehat\mallows$ of the Mallows kernel
is strictly positive definite at all irreducible representations
$\rho_\lambda$.
\end{theorem}

Note that Theorem~\ref{ThmMallows} corrects an assertion in the paper~\cite{jiao2015kendall}; 
the authors of that work suggested that since the
Mallows kernel depends only on the relative rankings of pairs of
objects, the Fourier transform $\widehat \mallows$ should be expected
to be zero at all irreps $\lambda \lhd (\numObj - 2 , 1, 1)$.
Theorem~\ref{ThmMallows} shows that this natural intuition does not
actually hold.

Theorem~\ref{ThmMallows} also has implications for the universality of
the Mallows kernel.   In \cite{gretton2012kernel}
the authors 
show that a universal and continuous kernel on a compact metric space
is characteristic---hence, a kernel on $\simg$ is universal if and only if it
is characteristic.  As with Theorem~\ref{ThmKendallTransform},
Theorem~\ref{ThmMallows} has implications for the kernel MMD induced
by the Mallows kernel. In particular, it shows that the Mallows kernel
is both characteristic and universal, and hence $\mmd_{\mallows}$ is a metric on
probability distributions over $\simg$.


\section{A family of polynomial-type kernels}
\label{SecPolyKernels}

Based on our results thus far, it is natural to suspect that there
exists a family of kernels interpolating between the relative
simplicity of the Kendall kernel, which is analogous to a linear
kernel on $\real^d$, and the richness of the Mallows kernel, which is 
analogous to a Gaussian kernel on $\real^d$.  This intuition motivates
us to introduce three families of polynomial-type
kernels on the symmetric group, defined as follows:
\begin{subequations}
\begin{align}
\polyKer{\power}(\sigma, \sigma') & \defn \left(1 +
\kendall(\sigma, \sigma')\right)^\power \\
\NorPolyKer{\power}(\sigma, \sigma') & \defn \left(1 +
\frac{\kendall(\sigma, \sigma')}{\power}\right)^\power, \quad
\mbox{and} \\
\LamPolyKer{\power}{\band}(\sigma, \sigma') & \defn
\exp\left({-\frac{\band}{2}{\numObj \choose 2}}\right) \left(1 + \band{\numObj
  \choose 2}\frac{\kendall(\sigma,
  \sigma')}{2\power}\right)^\power.
\end{align}
\end{subequations}
We refer to these three kernels as the \emph{polynomial
  kernel}, the \emph{normalized polynomial kernel}, and the
\emph{$\band$-normalized polynomial kernel} of degree $k$, respectively.  Since each
kernel depends only on the number of discordant pairs, they are all
right-invariant.  Moreover, each
kernel is positive semidefinite, since they can each be written as a
polynomial function of the Kendall kernel with
non-negative coefficients.

\begin{theorem}
\label{ThmPoly}
The Fourier transforms of the three polynomial kernels
$\polyKer{\power}$, $\NorPolyKer{\power}$,
$\LamPolyKer{\power}{\band}$ are zero at all irreducible
representations $\rho_\lambda$ with $\lambda \lhd (\max\{\numObj -
2\power, 1\}, 1, \ldots, 1)$. Furthermore, when $p~\geq~\numObj - 1$,
the Fourier transform of the three polynomial kernels is strictly positive
definite at all irreducible representations.
\end{theorem}
 
The first part of the theorem shows that the polynomial kernels of
degree $\power$ do not detect differences between distributions at
irreps $\rho_\lambda$ with $\lambda$ not higher in the partial
ordering than the \mbox{partition $(\max\{\numObj - 2\power, 1\}, 1,
  \ldots, 1)$.}  Intuitively, as the degree of the polynomial kernels
increases they are able to detect more differences between probability
distributions.  The second part of the theorem shows that the
polynomial kernels of degree at least $\numObj - 1$ detect all
differences between probability distributions.

The appeal of defining the second and third kernels, $\NorPolyKer{p}$
and $\LamPolyKer{p}{\band}$, in addition to the first one, is
two-fold. On the one hand, in practice, the kernel $\polyKer{p}$
becomes difficult to evaluate when $p$ is large because
$\polyKer{p}(\sigma, \sigma) = 2^p$. On the other hand, the two
normalized kernels satisfy the relations
\begin{align}
\label{EqPolyLim}
 \lim_{p \rightarrow \infty} \NorPolyKer{p}(\sigma, \sigma') &=
 \exp({\kendall(\sigma, \sigma')}) ~\mbox{~ and ~}  \lim_{p
   \rightarrow \infty} \LamPolyKer{p}{\band}(\sigma, \sigma') =
 \exp({-\band n_d(\sigma, \sigma')}) = \mallows (\sigma,
 \sigma').
\end{align}
The first limit is a constant times the Mallows kernel with the
parameter $\band = 2{\numObj \choose 2}^{-1}$, while the second limit
is the Mallows kernel $\mallows$. This observation suggests we can
infer properties about the Mallow's kernel by working with the
$\band$-normalized polynomial kernel. Indeed, our proof of
Theorem~\ref{ThmMallows} makes use of this fact.


\subsection{Feature maps of the polynomial kernels}

We now consider the feature spaces associated with the polynomial kernels.  We show
here how the dimensions of the feature spaces increase as the degree
of the kernels increases, eventually leading to the feature space of
the Mallows kernel (up to constants).  We give a recursive
construction of the feature maps \mbox{$\EmbdPoly{\power} \colon\simg
  \rightarrow \RR^{\left(1 + {\numObj \choose 2}\right)^\power}$} that
satisfy the relation \mbox{$\polyKer{\power}(\sigma, \sigma') =
  \EmbdPoly{\power}(\sigma)^\top \EmbdPoly{\power}(\sigma')$.}  First,
we use the feature map of the Kendall kernel to construct
$\EmbdPoly{1}$; in particular, the map \mbox{$\EmbdPoly{1} \colon
  \simg \rightarrow \RR^{1 + {\numObj \choose 2}}$} is defined by
\begin{align*}
\EmbdPoly{1}(\sigma)_{t_0} \defn 1 \quad \text{ and } \quad
\EmbdPoly{1}(\sigma)_{t_r} \defn \sqrt{{\numObj \choose
    2}^{-1}}\left(2\indi_{\{\sigma(i_r) < \sigma(j_r)\}} -1 \right),
\end{align*}
where the coordinates are indexed by the unordered pair $t_{0} = \{-1,
0\}$ and the ${\numObj \choose 2}$ unordered pairs $t_r = \{i_r,
j_r\}$ with $i_r, j_r \in [d]$ and $i_r < j_r$. We denote the set of
these unordered pairs by
\begin{align}
\label{EqTunordered}
\twosets \coloneqq \left\{t_0, t_1, \ldots, t_{{\numObj \choose
    2}}\right\}.
\end{align}
The feature map $\EmbdPoly{1}$ clearly satisfies $\polyKer{1} (\sigma,
\sigma') = \Phi_1(\sigma)^\top \Phi_1(\sigma')$. Now we use the map
$\EmbdPoly{\power -1}$ to construct a feature map $\EmbdPoly{\power}$
for $\power \geq 1$. By definition, we have
\begin{align*}
\polyKer{\power}(\sigma, \sigma') & = \left(1 + \kendall(\sigma,
\sigma')\right)\left(1 + \kendall(\sigma, \sigma')\right)^{\power - 1}
= \EmbdPoly{1}(\sigma)^\top \EmbdPoly{1}(\sigma')
\EmbdPoly{p-1}(\sigma')^\top \EmbdPoly{p-1}(\sigma) \\ 
& = \tr\left(\EmbdPoly{1}(\sigma)^\top \EmbdPoly{1} (\sigma')
\EmbdPoly{\power - 1}(\sigma')^\top \EmbdPoly{\power -
  1}(\sigma)\right)\\
& = \tr \left(\left(\EmbdPoly{1}(\sigma) \EmbdPoly{\power -
  1}(\sigma)^\top\right)^\top \EmbdPoly{1}(\sigma')\EmbdPoly{\power -
  1}(\sigma')^\top\right).
\end{align*}
Therefore, the polynomial kernel of degree $\power$ between $\sigma$
and $\sigma'$ is equal to the inner product of the matrices
$\Phi_{1}(\sigma)\Phi_{\power -1}(\sigma)^\top$ and
$\Phi_{1}(\sigma')\Phi_{\power -1}(\sigma')^\top$. By induction, we
see that $\EmbdPoly{\power}$ can be obtained from $\EmbdPoly{1}$ by
taking the outer product with itself $\power$ times, meaning that the
embedding $\Phi_\power: \simg \rightarrow \RR^{\left(1 + {\numObj
    \choose 2}\right)^\power}$ can be expressed in terms of a sequence
$s_1, s_2, \ldots, s_\power$ of elements of $\twosets$ as
\label{EqPolyMap}
\begin{align}
\label{eq:features_poly}
\EmbdPoly{\power}(\sigma)_{s_1 s_2 \ldots s_\power} = \prod_{i =
  1}^\power \EmbdPoly{1}(\sigma)_{s_i}.
\end{align}

It is clear that as the degree of the polynomial kernels increases,
the kernels capture more information about the probability
distribution of the data. 
Proposition~\ref{PropMallowsMap} together with
Eq.~\eqref{eq:features_poly} show that when the degree of the
polynomial kernels is at least ${\numObj \choose 2}$, their
feature sets contain all the features of the Mallows kernel (up to
constants). This offers another perspective on how the
polynomial kernels interpolate between the Kendall kernel and the
Mallows kernel.


\section{Empirical Results} 
\label{sec:experiments}

We now present an empirical exploration of our kernel-based methodology.  We
present results for simulated data and for two real-world datasets---the European
Union survey Eurobarometer data and the large-scale MovieLens dataset.

\subsection{Experiments with simulated data}
\label{sec:simulated_data}

\paragraph{Setup.} We evaluate the empirical power
of two-sample hypothesis tests based on the Kendall and Mallows kernel
$U$-statistics.  In order to do so, we chose pairs of probability
distributions $\probP$ and $\probQ$ over $\simg$ and then sampled
i.i.d. rankings $\alpha_1$, $\alpha_2$, \ldots, $a_{n}$ from $\probP$
and $\beta_1$, $\beta_2$, \ldots, $\beta_{n}$ from
$\probQ$.  The size of the rankings was fixed to $d = 5$.  The
hypothesis tests considered here reject the null when
$T_\kernel(\alpha, \beta) > t^*$ where the threshold $t^*$ is chosen
by permutation testing to ensure the probability of a false positive
is at most $0.05$.

We fixed $\probP$ to be the uniform distribution over $\simg$ and we chose $Q$ such that $\|A_\tau (\probP - \probQ)\| = \delta$ for different values of $\delta$, where $A_\tau$ is the matrix discussed in Proposition~\ref{PropKendallMap}. For each value of $\delta$ there are many distributions $\probQ$ 
that are at the prescribed distance from $\probP$.
When $\delta > 0$, we first sampled uniformly a direction from the complement of the null space of $A_\tau$ and then chose the distribution $\probQ$ at the prescribed distance away from $\probP$ in that direction. When $\delta = 0$, we sampled uniformly a direction from the null space of $A_\tau$ and then chose the distribution $Q$ which is the farthest away from $\probP$ in that direction. 

Once $\probP$ and $\probQ$ were fixed,  we sampled i.i.d. sets of rankings $\{\alpha_i\}_{i = 1}^n$ and $\{\beta_i\}_{i = 1}^n$ from the distributions $\probP$ and $\probQ$ respectively. We varied the sample size $n$ from $10$ to $300$ in increments of $10$. For each pair of sample sets $\{\alpha_i\}_{i = 1}^n$ and $\{\beta_i\}_{i = 1}^n$ we used $200$ permutations of these $2n$ data points to estimate the rejection threshold $t^*$. To get estimates of the power of the kernel tests, for each value of $n$, we sampled $1000$ data sets from the fixed distributions $\probP$ and $\probQ$ and ran the tests on them, measuring the frequency with which the tests rejected the null hypothesis.

\begin{figure}[h!]
\centering 
\begin{subfigure}[b]{0.49\textwidth}
\centerline{\includegraphics[width=0.85\columnwidth]{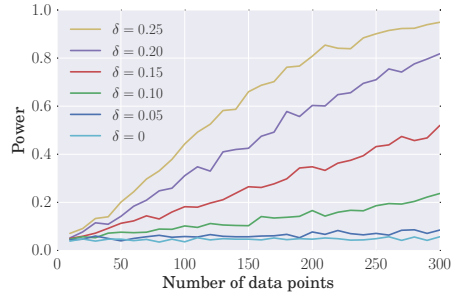}}
\caption{Power of Kendall kernel MMD.}
\label{FigKendall}
\end{subfigure}\\
\begin{subfigure}[b]{0.49\textwidth}
\centerline{\includegraphics[width=0.85\columnwidth]{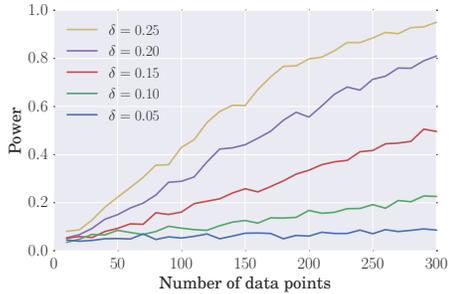}}
\caption{Power of Mallows' kernel MMD ($\band = 0.22$).}
\label{FigMallowsNonNull}
\end{subfigure}
\begin{subfigure}[b]{0.49\textwidth}
\centerline{\includegraphics[width=0.85\columnwidth]{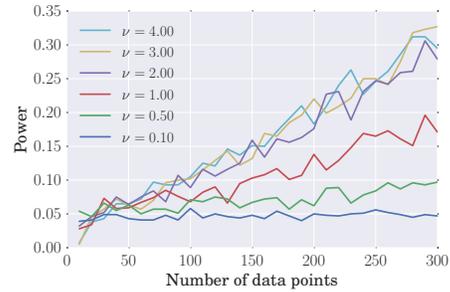}}
\caption{Power of Mallows' kernel MMD ($\delta = 0$).}
\label{FigMallowsNull}
\end{subfigure}
\caption{The empirical power of the MMD two-sample test with the Kendall kernel (a) and Mallows kernel (b,c) as a function of the number of data points $n$. For each $n$ we generated $1000$ data sets, and for each test we used $200$ permutations to choose the rejection threshold.}
\end{figure}

\paragraph{Discussion of results.} 

Recall the definition of the mean embedding $\mu_{\kernel, P} = \EE_{\sigma \sim \probP}\Phi(\sigma)$ of the probability distribution $P$ with respect to the kernel $\kernel$. Similarly, define the covariance matrix of $P$ as 
$
\Sigma_{\kernel,P} = \EE_{\sigma \sim \probP }\Phi(\sigma) \Phi(\sigma)^\top. 
$
Elementary computations~\citep{chen2010two}) show that
\begin{align*}
\EE T_\kernel(\alpha, \beta) &= \|\mu_{\kernel, \probP} - \mu_{\kernel, \probQ}\|^2 = 4{\numObj \choose 2}^{-1}\|A_\tau(\probP - \probQ)\|^2 \\
\var T_\kernel(\alpha, \beta) &=  \frac{2}{n(n - 1)}\tr\left(\Sigma_{\kernel, \probP}^2\right) + \frac{2}{n(n - 1)}\tr\left(\Sigma_{\kernel, \probQ}^2\right) + \frac{4}{n^2}\tr\left(\Sigma_{\kernel, \probP}\Sigma_{\kernel, \probQ}\right)\\
&\quad + \frac{4}{n}(\mu_{\kernel, \probP} - \mu_{\kernel, \probQ})^\top \Sigma_{\kernel, \probP}(\mu_{\kernel, \probP} - \mu_{\kernel, \probQ}) +  \frac{4}{n}(\mu_{\kernel, \probP} - \mu_{\kernel, \probQ})^\top \Sigma_{\kernel, \probQ}(\mu_{\kernel, \probP} - \mu_{\kernel, \probQ}).
\end{align*}

~\cite{ramdas2015adaptivity} showed that for real-valued data, when
$d$ and $n$ are sufficiently large, the power of kernel $U$-statistic
tests scales roughly like $\Psi (n \delta^2 / V)$ for sufficiently
small $\delta$, where $\Psi$ is the Gaussian CDF, and $V$ is a term
independent of $n$ which depends on the variance. The kernels over the
symmetric group do not satisfy the necessary assumptions to apply the
results of that work, but we observe a similar behavior in our
simulations.  For instance, Figure~\ref{FigKendall} shows the
empirical power of the Kendall kernel test as function of $n$ for
different values of $\delta$. As expected, the power of the test
increases as $\delta$ increases. More interestingly, observe that for
certain values of $n$ a doubling of $\delta$ translates into roughly
four times more power. Finally, note that when $\delta = 0$ the
Kendall kernel test has trivial power $0.05$. This behavior meets our
expectations based on Proposition~\ref{PropKendallMap} and the results
of~\cite{ramdas2015adaptivity}.

From Proposition~\ref{PropMallowsMap} we know that as the bandwidth
$\nu$ decreases, the weight of the features $\indi_{\{\sigma \colon
  \sigma(a) < \sigma(b)\}}$ increases relative to higher-order
features. Therefore, when $\delta > 0$ we expect that the Mallows
kernel with a small bandwidth will match the performance of the
Kendall kernel. Figure~\ref{FigMallowsNull} corroborates this
intuition---it shows the power of the Mallows kernel with bandwidth
$\band = 0.22$. When $\delta = 0$ the low-order features do not
capture the difference between $\probP$ and $\probQ$, and therefore a
higher bandwidth should yield higher power.  In
Figure~\ref{FigMallowsNull} we see that the Mallows kernel has power
against the null hypothesis $\probP = \probQ$ even when $\delta = 0$.
These results agree with the fact that the Mallows kernel is
characteristic (Theorem~\ref{ThmMallows}). As expected, when $\delta =
0$ a higher bandwidth yields more power, but at the cost of a higher
variance of the statistic.

\subsection{Survey Data}
\label{sec:survey_data}

\paragraph{Dataset and Methods.} In this section we showcase the use of kernels for hypothesis testing, classification, and regression on a real rankings dataset:
the European Union survey Eurobarometer $55.2$~\cite{eu200155}. As part of this survey, collected in $2001$ in countries members of the European Union, participants 
expressed their views on topics ranging from the single currency, agriculture, to science and technology. 
Participants were selected through a multi-stage stratified random sampling method, and there were $16130$ respondents in total. 
As part of the survey the participants were asked to rank in the order of preference six sources of news regarding scientific developments: TV, radio, newspapers and magazines, scientific magazines, the internet, school/university. The dataset also includes demographic information such as gender and age; a snippet of the dataset is shown in Table~\ref{tab:survey_data}. 

We removed all respondents who did not provide a complete ranking over the six sources of news, leaving $12216$ participants. 
Then, we split the dataset in two distinct ways: across gender, and across age groups ($40$ or younger and over $40$). Out of the $12216$ participants, $5915$ were men, $6301$ were women, $5985$ were $40$ 
or younger, $6231$ were over $40$. 

We ran two-sample hypothesis tests across these groups with
both the Kendall and the Mallows kernels. Furthermore, we fitted a kernel SVM with the Mallows kernel to predict the age group of participants. Finally, we fitted a kernel 
ridge  regression model with the Mallows kernel to predict the age of participants. For both the classification and regression tasks we used the Scikit-Learn Python package of~\cite{scikit-learn}
to fit the models. Then bandwidth of the Mallows kernel and the regularization parameter were chosen by cross-validation. 

\paragraph{Results and Discussion.}

For the hypothesis tests across gender we sub-sampled $300$ participants from each of the two groups and ran a permutation test with $400$ permutations, using the Kendall and the Mallows ($\band = 1$) kernel U-statistics. We obtained $p$-values equal to $0.075$ and $0.412$ respectively. After increasing the number of samples from each group to $600$, we obtained $p$-values equal to $0.002$ and $0.002$ respectively.  

For the hypothesis tests across age groups we sub-sampled $30$ participants from each of the two group and ran a permutation test with $400$ permutations, using the Kendall and the Mallows ($\band = 1$) kernel U-statistics. We obtained $p$-values equal to $0.007$ and $0.477$ respectively. After increasing the number of samples from each group to $50$, we obtained $p$-values equal to $0.002$ and $0.005$ respectively. 
We note that fewer samples than for the tests across gender were required to reject null hypothesis. For the type of rankings considered here we did expect a large discrepancy across age groups. 
In general young participants are more likely to attend schools or universities, making them more likely to rank highly these institutions as preferred source of information. Moreover, in $2001$,
it was to be expected that younger participants were more accustomed to the internet than older participants.

For the classification task across age groups, we fit a kernel SVM model using the Mallows kernel. We split the $12216$ participants randomly into a training set of $10000$ participants, and a 
test set of $2216$ participants. The bandwidth for the Mallows kernel was chosen to be $0.1$ through cross-validation. We obtained an error rate of $34\%$, which is better than chance. 

For the regression task to predict age, we fit a kernel ridge regression model using the Mallows kernel. We split the $12216$ participants randomly into a training set of $10000$ participants, and a 
test set of $2216$ participants. The bandwidth for the Mallows kernel was chosen to be $0.1$ through cross-validation.
 The model predicted the age of the respondents in the test set with an average $\ell_1$-error of $11$ years.

\subsection{Movie Ratings}
\label{sec:movie_data}

\paragraph{Dataset and Methods.}

Not all rankings come in the form of explicit orderings of
alternatives.  The MovieLens 1M dataset contains about one million
ratings of movies provided by $6000$ users of the website
\url{movielens.org}. For each user in the dataset we are given their
gender, age, and occupation, and for each movie we are given their
classifications into genres.  Each movie can belong to multiple genres
such as action, drama, thriller and comedy.  The movies contained in
this dataset are split into a total of $18$ genres. The ratings are
measured on a $5$-star scale.

For each movie genre we counted the number of movies belonging to that
genre, and then kept only the ratings to movies belonging to at least
one of the ten most popular genres. Then, for each user we computed
the average of the ratings across the ten movie genres. Finally, we
removed all users that did not record at least one rating for each of
the ten movie genres. The total number of users remaining in the
dataset after these procedures was $4428$.

We split this data in two distinct ways: across gender, and across age
groups (younger than $35$ and $35$ or older).  Given this data we ran
two-sample hypothesis tests across the groups by using the standard
linear kernel $U$-statistic for data in $\RR^{10}$. Furthermore, for
each user we transformed the average ratings into rankings in the
obvious way (the highest average rating takes rank one and so on),
breaking ties randomly.  Given the data in this new format, we ran
two-sample hypothesis tests across the groups by using the Mallows ($\band
= 0.2$) and the Kendall kernel U-statistics. For all three
$U$-statistics we sub-sampled $n$ samples from each group of users and
used $200$ permutations to determine the rejection threshold. For each
sample size $n$, we ran $100$ trials to estimate the empirical power
of the hypothesis tests.


\paragraph{Results and discussion.}

Our findings are summarized in Figure~\ref{FigMovies}.  All three
tests reject their respective null hypotheses with power going to one
as the number of data points used increases.  It is interesting to
note that depending on the split of the data, either the Kendall and
Mallows tests have more power than the linear kernel, or the other way
around. Of course, the linear kernel tests the equality of the
probabilities of the average ratings in $\RR^{10}$, whereas the
Kendall and the Mallows kernel are testing for differences in their
respective feature spaces.  This observation gives a strong
motivation for transforming ratings into rankings.  As compared to
studying the distributions of the users' ratings, it is arguably more
natural to study the users' distributions of preferences between movie
genres.

\begin{figure}
\centering
\begin{subfigure}[b]{0.49\textwidth}
\centerline{\includegraphics[width=0.85\columnwidth]{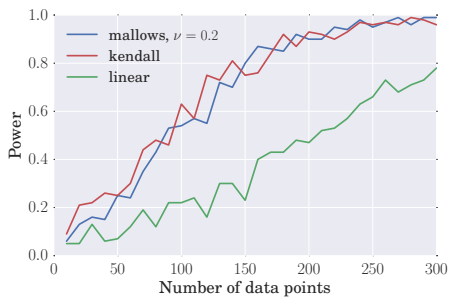}}
\caption{Hypothesis tests across gender.}
\label{FigMoviesGender}
\end{subfigure}
\begin{subfigure}[b]{0.49\textwidth}
\centerline{\includegraphics[width=0.85\columnwidth]{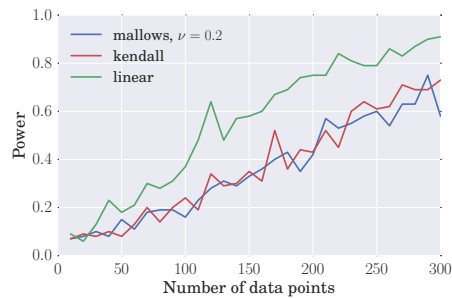}}
\caption{Hypothesis tests across age groups.}
\label{FigMoviesAge}
\end{subfigure}
\caption{The empirical power of the Mallows, Kendall, and linear kernel tests as a function of the number of data points $n$. For each $n$ we sub-sampled the dataset $100$ times to estimate the fraction of times the tests reject at level $0.05$. For each test we used $200$ permutations to select the rejection threshold.}
\label{FigMovies}
\end{figure}

\section{Conclusions}
\label{SecConc}

In this paper, we provided feature map and Fourier-analytic characterizations 
for the following right-invariant kernels: Kendall, Mallows, and a novel
family of polynomial kernels. We showed that the Kendall kernel is
nearly degenerate in two ways: its Gram matrix has rank ${\numObj \choose 2}$, and 
it has only two nonzero Fourier matrices, both of which have rank one. 
We constructed a $2^{\numObj \choose 2}$ feature map for the Mallows kernel and 
showed that its Gram matrix has full rank $\numObj!$.
This shows that the Mallows kernel lies at the other extreme than the Kendall kernel, 
being both universal and characteristic. In Fourier space, this translates to the 
Mallows kernel having a strictly positive definite Fourier transform at
all the irreps. Moreover, having the feature map for the Mallows kernel in closed form
informed our choice of bandwidth for the Mallows kernel in the two-sample
testing experiments. 

These results reveal that the Kendall and Mallows
kernels are quite different, even though both of them depend only
on counting discordant pairs between rankings.  There is a natural
analogy between these kernels in the space of permutations to the
linear and Gaussian kernels in Euclidean space. Building on this 
analogy, we proposed a new class of polynomial kernels that smoothly
interpolate between the Kendall and Mallows extremes, yielding a
hierachy of kernels that are sensitive to differences between
distributions at an increasingly dense set of frequencies.

Many properties of the Fourier transform of the Mallows and polynomial
kernels are still not understood. For example, unlike the case of the
Kendall kernel, we do \emph{not} have closed-form descriptions of the
Fourier matrices for these kernels.  Such concise expressions would
not only be of mathematical interest, but could also useful for
computing in the spectral domain.  It would also be interesting to
understand the properties of these kernels when applied to partial
rankings (top-$k$ or random-$k$), which is even harder because partial
rankings do not jointly form a group. We view the current results on
kernels for full rankings as an important step towards developing and
rigorously analyzing flexible kernel methods for partial rankings.

In Section~\ref{sec:experiments} we studied the empirical power of
kernel $U$-statistic two-sample tests with the Kendall and Mallows
kernels under different sets of alternatives. The scaling of the power
with the distance between the alternatives is similar to that of the
linear and Gaussian kernel over real data.  It would be interesting to
characterize the power of two-sample tests using the
Kendall or Mallows kernel as a function of the number of samples,
$\numObj$ and an appropriate signal-to-noise ratio.  Our final set of
experiments involved data transformed from raw numerical scores
(ratings of movies) into rankings, a transformation also explored
by~\cite{jiao2015kendall}.  This type of reduction to rank statistics,
while well studied in the context of classical rank-based methods for
testing~\citep{lehmann2006nonparametrics}, merits further study in the
context permutation-based covariates.  It offers invariance to
arbitrary monotone transformations of the covariates, and hence a way
of protecting against model mismatch and/or covariate biases.



\appendix 

\section*{Outline}

This supplementary material includes all the proofs of the results presented in
the main text. Sections~\ref{SecProofKendallMap} and~\ref{SecProofMallowsMap} contain the proofs of the results presented in 
Section~\ref{SecFeatures} of the main text. In Sections~\ref{SecProofKendall},~\ref{SecProofPoly}, and~\ref{SecProofMallows} we prove our results concerning the Fourier spectra 
of the Kendall, polynomial, and Mallows kernels. Section~\ref{SecAppRep} contains further background material on representation theory and Fourier analysis on the symmetric group needed in the proofs. Finally, Section~\ref{SecAppProofs} contains some proofs
of miscellaneous claims that are used throughout.  

\section{Proofs of main results}
\label{SecProofs}

\subsection{Proof of Proposition~\ref{PropKendallMap}}
\label{SecProofKendallMap}

We prove that all the basis vectors of $\RR^{{\numObj \choose 2}}$ are in the span of the matrix $A_\tau$.
To achieve this we order the tuples $t_i = (a_i, b_i)$, with $a_i < b_i$, as follows. The tuples $t_i$ and $t_j$ are ordered $t_i < t_j$ if and only if $a_i < a_j$ or $a_i = a_j$ and $b_i < b_j$. With total ordering fixed the $i$-th coordinate of $\RR^{{\numObj \choose 2}}$ corresponds to the tuple $t_i$. 

Then it is enough to prove that for any $1 \leq j \leq {\numObj \choose 2}$ the vector $v_j = \sum_{i = 1}^j e_{i}$ is equal to the column $(A_\tau)_\sigma$ for some appropriate $\sigma \in \simg$, where $e_{i}$ is the $i$-th standard basis vectors of $\RR^{{\numObj \choose 2}}$. 

We will construct inductively the permutations $\pi_j$ such that $(A_\tau)_{\pi_j} = v_j$. Observe that the identity permutation $\pi_{{d\choose 2}}(i) = i$ satisfies $(A_{\tau})_{\pi_{{d\choose 2}}} = v_{{d\choose 2}}$. Also, note that if we swap the ranks of ${\numObj}$ and $\numObj - 1$ in the permutation $\pi_{{\numObj \choose 2}}$, we obtain a permutation $\pi_{{d\choose 2} - 1}$ such that $(A_\tau)_{\pi_{{d\choose 2} - 1}} = v_{{d\choose 2} - 1}$.

Assume we have constructed $\pi_{j+1}$ such that $(A_\tau)_{\pi_{j + 1}} = v_{j+1}$. We construct $\pi_j$ such that $(A_\tau)_{\pi_{j}} = v_j$. Let $t_{j + 1} = (a, b)$ be the tuple corresponding to the $j+1$-st coordinate. Since $(A_\tau)_{\pi_{j+1}} = v_{j+1}$, we have $\pi_{j+1}(a) < \pi(r)$ for all $a < r \leq b$, and $\pi_{j+1}(a) > \pi_{j + 1}(r)$ for all $r > b$. Moreover $\pi_{j+1}(r) > \pi_{j + 1}(r + 1)$ for all $a < r < b$. Therefore, if we choose $\pi_{j}(r) = \pi_{j + 1}(r)$ for all $r$ distinct from $a$ and $b$ and $\pi_j (a) = \pi_{j + 1}(b)$, $\pi_j(b) = \pi_{j + 1}(a)$, we find that $(A_\tau)_{\pi_j} = v_j$. The conclusion follows.

\subsection{Proof of Proposition~\ref{PropMallowsMap}}
\label{SecProofMallowsMap}

We begin with the observation that the $\band$-normalized polynomial kernel converges to the Mallows kernel as its degree increases:
\begin{align*}
\LamPolyKer{\power}{\band}(\sigma, \sigma') = e^{-\frac{\band}{2}{\numObj \choose 2}} \left(1 + \band{\numObj
  \choose 2}\frac{\kendall(\sigma,
  \sigma')}{2\power}\right)^\power \xrightarrow{\power \rightarrow \infty} e^{-\band n_{\numObj}(\sigma, \sigma')}=\mallows(\sigma, \sigma').
\end{align*}
We construct a feature map for the Mallows kernel by exploiting this observation. Specifically, we derive a feature map of the $\band$-normalized polynomial kernel and compute its limit as the degree $\power$ of the kernel goes to infinity. 
Similar to Section~\ref{SecFeatures}, we define the feature map $\EmbdNor{}~\colon~\simg~\to~\RR^{{\numObj \choose 2} + 1}$:
\begin{align*}
\EmbdNor{}(\sigma)_{t_r} \defn 2\indi_{\{\sigma(a_r) < \sigma(b_r)\}} -1,
\end{align*}
where the coordinates are indexed by the ordered pairs $t_r = \{a_r,
b_r\}$ with $a_r, b_r \in [\numObj]$ and $a_r < b_r$. 
Let $\twosets^*$ denote the set of ${\numObj \choose 2}$ such tuples. Then, a binomial expansion yields
\begin{align*}
\LamPolyKer{\power}{\band}(\sigma, \sigma') &= e^{-\frac{\band}{2}{\numObj \choose 2}}\left(1 + \frac{\band}{2\power}\sum_{i = 1}^{{\numObj \choose 2}} \EmbdNor{}(\sigma)_{t_i}\EmbdNor{}(\sigma')_{t_i}\right)^\power\\
&= e^{-\frac{\band}{2}{\numObj \choose 2}}\sum_{c_0 + \ldots + c_{\numObj \choose 2} = \power} \frac{\power !}{c_0! c_1!\ldots c_{{\numObj\choose 2}}!}\left(\frac{\band}{2\power}\right)^{\power - c_0} \prod_{i = 1}^{\numObj \choose 2}\EmbdNor{}(\sigma)_{t_i}^{c_i} \prod_{j = 1}^{\numObj \choose 2}\EmbdNor{}(\sigma')_{t_j}^{c_j}. 
\end{align*}

Note that $(\EmbdNor{}(\sigma))_{t_i}^2 = 1$ for any $t_i \in \twosets^*$ and any $\sigma \in \simg$. For any $A \subset \twosets^*$ we denote $\EmbdNor{}(\sigma)_A \defn \prod_{t_i \in A} \EmbdNor{}(\sigma)_{t_i}$, and $\EmbdNor{}(\sigma)_\emptyset = 1$. Hence, we can simplify the above expression to
\begin{align*}
\LamPolyKer{\power}{\band}(\sigma, \sigma') &= e^{-\frac{\band}{2}{\numObj \choose 2}}\sum_{A \subset \twosets^*} \EmbdNor{}(\sigma)_A \EmbdNor{}(\sigma')_A\sum_{\substack{c_0 + \ldots + c_{\numObj \choose 2} = \power \\ c_i \text{ odd when } t_i \in A \\ c_i \text{ even when } t_i \not \in A}} \frac{\power !}{c_0! c_1!\ldots c_{{\numObj\choose 2}}!}\left(\frac{\band}{2\power}\right)^{\power - c_0}.
\end{align*}

By symmetry the second sum on the right hand side depends only on the power $\power$ and the size of the set $A$. Therefore, if we define 
\begin{align*}
\delta(\power, r) = \sum_{\substack{c_0 + \ldots + c_{\numObj \choose 2} = \power \\ c_i \text{ odd when } 1\leq i \leq r \\ c_i \text{ even when } r < i}} \frac{\power !}{c_0! c_1!\ldots c_{{\numObj\choose 2}}!}\left(\frac{\band}{2\power}\right)^{\power - c_0},
\end{align*}
we find that 
\begin{align*}
\LamPolyKer{\power}{\band}(\sigma, \sigma') &= e^{-\frac{\band}{2}{\numObj \choose 2}}\sum_{A \subset \twosets^*} \EmbdNor{}(\sigma)_A \EmbdNor{}(\sigma')_A \delta(\power, |A|).
\end{align*}

We are left to compute the limit of $\delta(\power, |A|)$ as $\power \rightarrow \infty$, and to this end  we construct a generating function for the sequence $\delta(\power, r)$ by defining
\begin{align*}
F(z) &:= \power! \left(\frac{\band}{2\power}\right)^\power e^{\frac{2\power}{\band}z} \left(\frac{e^z - e^{-z}}{2}\right)^r \left(\frac{e^z + e^{-z}}{2}\right)^{{\numObj \choose 2} - r}.
\end{align*}
 By Taylor expanding each term $e^z$ individually we see that the function $F(z)$ is the generating function of $d(\power, r)$ for $0 \leq r \leq {\numObj \choose 2}$. More precisely, $\delta(\power, r)$ is the  $p$-th coefficient of the generating function $F(z)$. By expanding $F(z)$ into a linear combination of exponentials we can compute $\delta(\power, r)$ and study its asymptotic behavior. 
\begin{align*}
F(z) &= \power! \left(\frac{\band}{2\power}\right)^\power \frac{e^{\frac{2\power}{\band}z - {\numObj \choose 2}z}}{2^{\numObj \choose 2}} \left(e^{2z} - 1\right)^r \left( e^{2z} + 1\right)^{{\numObj \choose 2} - r}\\
&= \frac{\power!}{2^{\numObj \choose 2}} \left(\frac{\band}{2\power}\right)^\power e^{\frac{2\power}{\band}z - {\numObj \choose 2}z}\sum_{i = 0}^{r}\sum_{j = 0}^{{\numObj \choose 2} - r}e^{2(i + j)z} (-1)^{r - i}{r\choose i}{{\numObj \choose 2} - r \choose j}\\
&= \frac{\power!}{2^{\numObj \choose 2}} \left(\frac{\band}{2\power}\right)^\power\sum_{i = 0}^{r}\sum_{j = 0}^{{\numObj \choose 2} - r} e^{\left(\frac{2\power}{\band} - {\numObj \choose 2} + 2(i + j)\right)z} (-1)^{r - i}{r\choose i}{{\numObj \choose 2} - r \choose j}
\end{align*}

Therefore,
\begin{align*}
\delta(\power, r) &= \frac{1}{2^{\numObj \choose 2}} \left(\frac{\band}{2\power}\right)^\power\sum_{i = 0}^{r}\sum_{j = 0}^{{\numObj \choose 2} - r}\left(\frac{2\power}{\band} - {\numObj \choose 2} + 2(i + j)\right)^\power (-1)^{r - i}{r\choose i}{{\numObj \choose 2} - r \choose j}\\
&= \frac{1}{2^{\numObj \choose 2}}\sum_{i = 0}^{r}\sum_{j = 0}^{{\numObj \choose 2} - r}\left(1 - \frac{\frac{\band}{2}\left({\numObj \choose 2} - 2(i + j)\right)}{\power}\right)^\power (-1)^{r - i}{r\choose i}{{\numObj \choose 2} - r \choose j}\\
&\xrightarrow{\power \rightarrow \infty} \frac{1}{2^{\numObj \choose 2}}\sum_{i = 0}^{r}\sum_{j = 0}^{{\numObj \choose 2} - r}e^{-\frac{\band}{2}\left({\numObj \choose 2} - 2(i + j)\right)} (-1)^{r - i}{r\choose i}{{\numObj \choose 2} - r \choose j}\\
&=  \frac{e^{-\frac{\band}{2}{\numObj \choose 2}}}{2^{\numObj \choose 2}}\sum_{i = 0}^{r}\sum_{j = 0}^{{\numObj \choose 2} - r}e^{\band i} e^{ \band j}(-1)^{r - i}{r\choose i}{{\numObj \choose 2} - r \choose j}\\
&=\frac{e^{-\frac{\band}{2}{\numObj \choose 2}}}{2^{\numObj \choose 2}} \left(e^\band - 1\right)^r \left(e^\band + 1\right)^{{\numObj \choose 2} - r}. 
\end{align*}

The conclusion follows.


\subsection{Proof of Theorem~\ref{ThmKendallTransform}}
\label{SecProofKendall}

For $d=2$ and $d=3$, the irreps $\rho_\lambda$ are easy to describe in
closed-form~\cite{diaconis1988group}; in particular, we have
\begin{align*}
& \widehat \kendall(\rho_{(2)}) = 0, \quad
  \widehat\kendall(\rho_{(1,1)}) = 2\\ &\widehat \kendall(\rho_{(3)})
  = 0,\quad \widehat\kendall(\rho_{(2,1)}) = 
\begin{pmatrix}
\frac{2}{3} & \frac{2}{\sqrt{3}}\\ \frac{2}{\sqrt{3}} & 2
\end{pmatrix},
\quad \widehat\kendall(\rho_{(1,1,1)}) = \frac{2}{3}.
\end{align*}

Accordingly, it remains to prove Theorem~\ref{ThmKendallTransform}
when $d \geq 4$.  Each representation $\rho \colon \simg \rightarrow
\CC^{d_\rho \times d_\rho}$ defines a collection of $d_\rho^2$
functions $\sigma \mapsto \rho(\sigma)_{ij}$ on the symmetric
group. An important result in the representation theory states that
the functions defined by the irreps $\rho_\lambda$ form a basis for
the space of functions over the symmetric group. To exploit this 
fact, we express the Kendall kernel as a linear combination of the 
functions defined by the overcomplete representation $\tau_{(d - 2, 1, 1)}$
defined in Section~\ref{SecAppRep}, equation~\ref{EqOverCompDefn}. 

\begin{lemma}
\label{LemKendallStepI}
The Kendall function $\sigma \mapsto \kendall(\sigma)$ is a linear combination of the functions defined by the representation $\tau_{(d - 2, 1, 1))}$.
\end{lemma}

\begin{proof}
A rough sketch of the argument is as follows. The Kendall function is a linear combination of indicator functions
$\indi_{\{\sigma(i) > \sigma(j)\}}$ (plus a constant). The result follows because
each of these functions is a linear combination of the indicator functions $\indi_{\{\sigma(i) = l, \sigma(j) = r\}}$, which are exactly the functions defined by $\tau_{(\numObj - 2, 1, 1)}$. 

Formally, to prove the
claim it suffices to express the function $\kernel_\tau$ as a linear
combination of the functions defined by $\rho_{(d)}$,
$\rho_{(d-1,1)}$, $\rho_{(d-2,2)}$, and $\rho_{(d-2,1,1)}$. James'
submodule theorem states that
\begin{align*}
\tau_{(d-2,1,1)} \equiv \rho_{(d)}\oplus \rho_{(d-1,1)} \oplus
\rho_{(d-1,1)} \oplus \rho_{(d-2,2)} \oplus \rho_{(d-2,1,1)}.
\end{align*}
Therefore, we just have to show that the Kendall function is a linear
combination of the functions defined by $\tau_{(d-2,1,1)}$. We have
\begin{align*}
k_\tau(\sigma) = 1 - 2\frac{\ii(\sigma)}{{d\choose 2}} & = 1 - 2{d
  \choose 2}^{-1} \sum_{i<j} \indi_{\{\sigma(i) > \sigma(j)\}}\\ 
& = 1 - 2{d\choose 2}^{-1} \sum_{i<j} \sum_{l < r}\indi_{\{\sigma(i) =
  r, \sigma(j) = l\}}.
\end{align*}
The functions $\indi_{\{\sigma(i) = r, \sigma(j) = l\}}$ are defined
by $\tau_{(\numObj - 2,1,1)}$ by construction, completing the proof.
\end{proof}

Our next step in proving Theorem~\ref{ThmKendallTransform} is to
compute the Fourier transform of the Kendall transform at the
representations $\tau_{(\numObj)}$, $\tau_{(\numObj - 1, 1)}$,
$\tau_{(\numObj - 2, 2)}$, and $\tau_{(\numObj - 2, 1, 1)}$.  Our next
lemma summarizes the results of these computations, which though
technical are conceptually straightforward. First define
\begin{itemize}
\item vector $u \in \RR^\usedim$ with entries $u_i = \numObj - 2i + 1$
  for $i = 1, \ldots, \numObj$
\item vector $w \in \RR^{{\numObj \choose 2}}$ with entries $w_{\{i,
  j\}} = 2d - 2(i + j) + 2$ for $1 \leq i < j \leq \numObj$.
\item vectors $v_{1}$, $v_2$, and $v_3$ in $\RR^{\numObj(\numObj -
  1)}$ with entries
\begin{align*}
[v_1]_{(i, j)} &= 1 - 2\indi_{\{i > j\}},\\
[v_2]_{(i, j)} &= d - 2i + 2 - 2\indi_{\{i < j\}},\\
[v_3]_{(i, j)} &= d - 2j + 2 - 2\indi_{\{i > j\}}.
\end{align*}
\end{itemize}
With this notation, we have the following:
\begin{lemma}
\label{LemKendallStepII}
The Fourier transform of the Kendall kernel satisfies the identities:
\begin{subequations}
\begin{align}
\widehat\kendall (\tau_{(\numObj)}) & = 0\text{, } \quad
\widehat\kendall(\tau_{(\numObj - 1, 1)}) = \frac{(\numObj
  -2)!}{{\numObj \choose 2}}uu^\top, \\
 \widehat\kendall(\tau_{(\numObj - 2, 2)}) & = \frac{(\numObj -
   3)!}{{\numObj \choose 2}}ww^\top, \qquad \text{and} \\
\widehat\kendall(\tau_{(\numObj - 2, 1, 1)}) & = \frac{(\numObj -
  2)!}{{\numObj \choose 2}}v_1 v_1^\top + \frac{(\numObj -
  3)!}{{\numObj \choose 2}} v_2 v_2^\top + \frac{(\numObj -
  3)!}{{\numObj \choose 2}} v_3 v_3^\top.
\end{align}
\end{subequations}
\end{lemma}
\begin{proof}
Lemma~\ref{LemKendallStepII} states in closed form the values of
$\widehat\kernel_\tau$ evaluated at the four representations
$\tau_{(\numObj)}$, $\tau_{(\numObj - 1, 1)}$, $\tau_{(\numObj - 2,
  2)}$, and $\tau_{(\numObj - 2, 1, 1)}$. We compute these values one
at a time.

\paragraph{Computing $\widehat\kernel_\tau(\tau_{(\numObj)})$.}
We first show that $\widehat\kernel_\tau(\tau_{(\numObj)}) = 0$.
Recall that $\tau_{(d)}$ is the trivial representation, equal to $1$
at all permutations, so we need to check that $\sum_{\sigma \in \simg}
1 - 2{\numObj\choose 2}^{-1}\ii(\sigma) = 0$. Note that we have
\begin{align*}
\sum_{\sigma \in\simg} \ii(\sigma)= \sum_{\sigma \in \simg} \sum_{i<j}
\indi_{\{\sigma(i) > \sigma(j)\}} = \sum_{i<j} \sum_{\sigma \in \simg}
\indi_{\{\sigma(i) > \sigma(j)\}} = \sum_{i< j} \frac{\numObj!}{2} =
\frac{\numObj!{\numObj \choose 2}}{2},
\end{align*}
so that the conclusion follows.


\paragraph{Computing $\widehat \kernel_\tau(\tau_{(\numObj - 1, 1)})$.}
In this case, we show that $\widehat\kernel_\tau(\tau_{(\numObj - 1,
  1)}) = \frac{(\numObj -2)!}{{\numObj \choose 2}}vv^\top$, where the
vector $v \in \RR^\numObj$ has components $v_r = \numObj - 2r + 1$.
Consider the functions $g_{ij}$ on $\simg$ defined by $g_{ij}(\sigma)
= 1 - 2\indi_{\{\sigma(i)> \sigma(j)\}}$, for all $i < j$. Then
\begin{align*}
\kernel_\tau(\sigma) = \frac{1}{{\numObj\choose 2}}\sum_{i<j}
g_{ij}(\sigma)\; \text{ and hence }\; \widehat\kernel(\rho) =
\frac{1}{{\numObj \choose 2}}\sum_{i<j} \widehat g_{ij}(\rho),
\end{align*}
for any representation $\rho$. 

We compute $\widehat g_{ij}(\tau_{(\numObj-2,2)})$ for each tuple $i <
j$ and then sum up the results. The rows of $\tau_{(\numObj-1,1)}$ are
indexed by tabloids of shape $(d-1,1)$. Each of these tabloids is
fully specified by the index contained in the second row.  We identify
the tabloids of shape $(\numObj-1,1)$ with those indices. Let $t_1$
and $t_2$ be two indices in $[d]$. Then
\begin{align*}
\widehat g_{ij}(\tau_{(\numObj-1,1)})_{t_1 t_2} = \sum_{\sigma \in
  \simg} \left(1 - 2\indi_{\{\sigma(i) >
  \sigma(j)\}}\right)\indi_{\{\sigma(t_{1}) = t_{2}\}}
\end{align*}

There are three cases to consider. First, suppose that $t_1$ is
distinct from both $i$ and $j$. There are $(d-1)!$ permutations
$\sigma$ that satisfy $\sigma(t_{1}) = t_{2}$, out of which exactly
half satisfy $g_{ij}(\sigma) = 1$ and the other half satisfy
$g_{ij}(\sigma) = -1$.  Therefore, we are guaranteed that $\widehat
g_{ij}(\sigma)_{t_1 t_2} = 0$ when $t_{1} \not \in \{i,j\}$.

Otherwise, we may assume that $t_{1} = i$. Then, out of the $(\numObj
- 1)!$ permutations that satisfy $\sigma(i) = t_{2}$ there are $(t_{2}
- 1)(\numObj - 2)!$ permutations that satisfy $\sigma(i) > \sigma(j)$
and $(\numObj - t_{2})(\numObj - 2)!$ that satisfy the opposite
inequality. Hence, $\widehat g_{ij}(\tau_{(d-1,1)}) = (\numObj -
2t_{2} + 1)(\numObj - 2)!$ when $t_{1}= i$.

The remaining (third) case is when $t_{1} = j$. Then, out of the
$(\numObj - 1)!$ permutations with $\sigma(j) = t_{2}$ there are
$(\numObj - t_{2})(\numObj - 2)!$ with $\sigma(i)> \sigma(j)$ and
$(t_{12}- 1)(\numObj - 2)!$ with $\sigma(i)< \sigma(j)$. Therefore
$\widehat g_{ij}(\tau_{(\numObj - 1,1)}) = -(\numObj - 2\tau_{2} +
1)(\numObj - 2)!$ when $t_{1} = j$. To summarize, we have
\begin{align*}
\widehat g_{ij}(\tau_{(\numObj - 1,1)})_{t_{1}t_{2}} =
\begin{cases}
0 & \text{ if } t_{1} \not\in \{i,j\}\\
(\numObj - 2t_{2} + 1)(\numObj - 2)! & \text{ if } t_{1} = i\\
(2t_{2} - \numObj - 1)(\numObj-2)! & \text{ if } t_{1} =j
\end{cases}.
\end{align*}

Now we need to sum the Fourier transforms of the functions $g_{ij}$ to
obtain the Fourier transform of $\kernel_\tau$. We have
\begin{align*}
\widehat\kernel_\tau(\tau_{(\numObj - 1, 1)})_{t_1 t_2} &= {\numObj
  \choose 2}^{-1}\sum_{i<j}\widehat g_{ij}(\tau_{(\numObj - 1, 1)}) \\
& = {\numObj \choose 2}^{-1} \sum_{t_1 = i < j}(\numObj - 2t_2 +
1)(\numObj - 2)! + {\numObj \choose 2}^{-1}\sum_{i<j = t_1}(2t_2 -
\numObj -1)(\numObj - 2)! \\
& = {\numObj \choose 2}^{-1}(\numObj - t_1)(\numObj - 2t_2 +
1)(\numObj - 2)!  + {\numObj \choose 2}^{-1}(t_1 - 1)(2t_2 -\numObj -
1)(\numObj - 2)!\\
& = {\numObj \choose 2}^{-1}(\numObj - 2t_1 + 1)(\numObj - 2t_2 +
1)(\numObj - 2)!,
\end{align*}
as claimed.


\paragraph{Computing $\widehat\kernel_\tau(\tau_{(\numObj - 2, 2)})$.}
In this case, we show that $\widehat\kernel_\tau(\tau_{(\numObj - 2,
  2)}) = \frac{(\numObj - 3)!}{{\numObj \choose 2}}ww^\top$, where the
vector $w \in \RR^{{\numObj \choose 2}}$ has entries $w_{\{r_1, r_2\}}
= 2d - 2(r_1 + r_2) + 2$ for $1 \leq r_1 < r_2 \leq \numObj$.

The entries of $\widehat\kernel_\tau(\tau_{(\numObj - 2, 2)})$ are
indexed by tabloids of shape $(\numObj - 2, 2)$ which can be
identified with the set of two indices contained in the second
row. Therefore we can identify the tabloids of shape $(\numObj - 2,2)$
with sets of two indices.  Fix two such sets $t_1 = \{t_{11},
t_{12}\}$ and $t_{2} =\{t_{21}, t_{22}\}$. Once again we use the
functions $g_{ij}(\sigma) \defn 1 - 2\indi_{\{\sigma(i) >
  \sigma(j)\}}$.  For these functions, we have
\begin{align*}
\widehat g_{ij}(\tau_{(\numObj - 2, 2)})_{t_1 t_2} & = \sum_{\sigma
  \in \simg} g_{ij}(\sigma) \indi_{\{\sigma(\{t_{11},t_{12}\}) =
  \{t_{21}, t_{22}\}\}} \\ 
& = \sum_{\sigma \in \simg}\left(1 -
2\indi_{\{\sigma(i)>\sigma(j)\}}\right)\left(\indi_{\{\sigma(t_{11}) =
  t_{21}, \sigma(t_{12}) = t_{22}\}} + \indi_{\{\sigma(t_{11}) =
  t_{22}, \sigma(t_{12}) = t_{21}\}}\right).
\end{align*}
By breaking into four cases, similar to the proof the computation of
$\widehat\kernel_\tau(\tau_{(\numObj - 2, 2)})$, we obtain
\begin{align*}
\widehat g_{ij}(\tau_{(d-2,2)})_{t_{1} t_{2}} =
\begin{cases}
0 &\text{ if } \{t_{11},t_{12}\} \cap \{i,j\} = \emptyset\\ 0 &\text{
  if } \{t_{11}, t_{12}\} = \{i,j\}\\ (2d - 2(t_{21} + t_{22}) +
2)(d-3)! &\text{ if }\{t_{11}, t_{12}\} \cap \{i\} = \{i\}\\ (2(t_{21}
+ t_{22})- 2d - 2)(d-3)! &\text{ if }\{t_{11}, t_{12}\} \cap \{i\} =
\{j\}
\end{cases}.
\end{align*}
Summing the terms $\widehat g_{ij}(\tau_{(\numObj - 2, 2)})$ over
pairs $i < j$ yields the result.


\paragraph{Computing $\widehat\kernel_\tau(\tau_{(\numObj - 2, 1, 1)})$.}
We show that
\begin{align*}
\widehat\kernel_\tau(\tau_{(\numObj - 2, 1, 1)}) = \frac{(\numObj -
  2)!}{{\numObj \choose 2}}v_1 v_1^\top + \frac{(\numObj -
  3)!}{{\numObj \choose 2}} v_2 v_2^\top + \frac{(\numObj -
  3)!}{{\numObj \choose 2}} v_3 v_3^\top,
\end{align*} 
where $v_{1}$, $v_2$, and $v_3$, are the vectors in
$\RR^{\numObj(\numObj - 1)}$ defined by
\begin{align*}
[v_1]_{(r_1, r_2)} &= 1 - 2\indi_{\{r_1 > r_2\}},\\ [v_2]_{(r_1, r_2)}
&= d - 2r_1 + 2 - 2\indi_{\{r_1 < r_2\}},\\ [v_3]_{(r_1, r_2)} &= d -
2r_2 + 2 - 2\indi_{\{r_1 > r_2\}}.
\end{align*}

The same ideas used in the computation of
$\widehat\kernel_\tau(\tau_{(\numObj - 2, 2)})$ apply here as
well. However, the analysis is a bit more detailed because there are
more cases to consider. The entries of
$\widehat\kernel_\tau(\tau_{(\numObj - 2, 1, 1)})$ are indexed by
tabloids of shape $(\numObj - 2, 1, 1)$. These tabloids are completely
specified by the entries contained in the second and third
rows. Hence, we can identify them with ordered tuples in
$[d]^2$. Fixing two such tuples $t_1 = (t_{11}, t_{12})$ and $t_2 =
(t_{21}, t_{22})$, with $t_{11} \neq t_{12}$ and $t_{21} \neq t_{22}$,
we then have
\begin{align*}
\widehat g_{ij}(\tau_{(\numObj - 2, 1, 1)})_{t_1 t_2} &= \sum_{\sigma
  \in \simg} g_{ij}(\sigma) \indi_{\{\sigma(t_{11}) = t_{21},
  \sigma(t_{12}) = t_{22}\}}  \\
& = \sum_{\sigma \in \simg}\left(1 -
2\indi_{\{\sigma(i)>\sigma(j)\}}\right)\indi_{\{\sigma(t_{11}) =
  t_{21}, \sigma(t_{12}) = t_{22}\}}.
\end{align*}
Arguments similar to the ones used in the computation of $\widehat
g_{ij}(\tau_{(\numObj - 1, 1)})$ enable us to compute $\widehat
g_{ij}(\tau_{(\numObj - 2, 1, 1)})$ as well. In order to make the
result more readable, let us split them into to cases: $t_{21} <
t_{22}$ and $t_{21} > t_{22}$. Then we obtain
\begin{subequations}
\begin{align}
\widehat g_{ij}(\tau_{(\numObj - 2, 1, 1)})_{t_{1} t_{2}}  & =
\begin{cases}
0 &\text{ if } \{t_{11},t_{12}\} \cap \{i,j\} = \emptyset\\
 (\numObj - 2)! &\text{ if } t_{11} = i,\; t_{12} = j,\; t_{21} <
t_{22}\\
-(\numObj - 2)! &\text{ if } t_{11} = j,\; t_{12} = i,\; t_{21} <
t_{22} \\
(\numObj - 2t_{21})(\numObj - 3)! &\text{ if } t_{11} = i,\;
t_{12}\neq j,\; t_{21} < t_{22}\\ 
(\numObj - 2t_{22} + 2)(\numObj - 3)! &\text{ if } t_{11} \neq j,\;
t_{12} = i,\; t_{21} < t_{22}\\ 
(2t_{21} - \numObj)(\numObj - 3)!  &\text{ if } t_{11} = j,\;
t_{12}\neq i,\; t_{21} < t_{22}\\ 
(2t_{22} - \numObj - 2)(\numObj - 3)! &\text{ if } t_{11} \neq i,\;
t_{12} = j,\; t_{21} < t_{22}\\
\end{cases}.
\end{align}
\begin{align}
\widehat g_{ij}(\tau_{(\numObj - 2, 1, 1)})_{t_{1} t_{2}} & =
\begin{cases}
0 &\text{ if } \{t_{11},t_{12}\} \cap \{i,j\} = \emptyset\\ -(\numObj
- 2)! &\text{ if } t_{11} = i,\; t_{12} = j,\; t_{21} >
t_{22}\\ (\numObj - 2)! &\text{ if } t_{11} = j,\; t_{12} = i,\;
t_{21} > t_{22}\\ (\numObj - 2t_{21} + 2)(\numObj - 3)! &\text{ if }
t_{11} = i,\; t_{12}\neq j,\; t_{21} > t_{22}\\ (\numObj -
2t_{22})(\numObj - 3)! &\text{ if } t_{11} \neq j,\; t_{12} = i,\;
t_{21} > t_{22}\\ (2t_{21} - \numObj - 2)(\numObj - 3)! &\text{ if }
t_{11} = j,\; t_{12}\neq i,\; t_{21} > t_{22}\\ (2t_{22} -
\numObj)(\numObj - 3)! &\text{ if } t_{11} \neq i,\; t_{12} = j,\;
t_{21} > t_{22}\\
\end{cases}.
\end{align}
\end{subequations}
The conclusion then follows by computing the sum $\sum_{i < j} g_{ij}
(\tau_{(\numObj - 2, 1, 1)})_{t_1 t_2}$ in the four possible cases
obtained from the orderings of $t_{11}$ and $t_{12}$, and of $t_{21}$
and $t_{22}$.
\end{proof}

At this point, Theorem~\ref{ThmKendallTransform} follows from
Lemma~\ref{LemKendallStepI} and Lemma~\ref{LemKendallStepII}. 
Lemma~\ref{LemKendallStepI} together with decomposition~\eqref{EqTauIV}
imply that $\widehat\kendall(\rho_\lambda) = 0$ for all $\lambda \lhd (\numObj -2, 1, 1)$
because the functions defined by irreps are a basis for the space of functions on $\simg$. 
Next, observe that
$\widehat\kendall(\tau_{(\numObj)}) = 0$ is equivalent to
$\widehat \kendall(\rho_{(\numObj)}) = 0$. Then, since the matrix
$\widehat\kendall(\tau_{(\numObj - 1, 1)})$ has rank one, the
decomposition~\eqref{EqTauII} of the representation $\tau_{(\numObj -
  1, 1)}$ implies that $\widehat\kendall(\rho_{(\numObj - 1, 1)})$
has rank one as well.  Furthermore, since both matrices
$\widehat\kendall(\tau_{(\numObj - 1, 1)})$ and
$\widehat\kendall(\tau_{(\numObj - 2, 2)})$ have rank one, from
decomposition~\eqref{EqTauIII} of the representation $\tau_{(\numObj -
  2, 2)}$ we obtain $\widehat\kendall(\rho_{(\numObj - 2, 2)}) =
0$. Finally, since the matrix $\widehat\kendall(\tau_{(\numObj -
  2, 1, 1)})$ has rank three, from decomposition~\eqref{EqTauIV} of
the representation $\tau_{(\numObj - 2, 1, 1)}$ we know that
$\widehat\kendall(\rho_{(\numObj - 2, 1, 1)})$ has rank one, which
completes the proof of Theorem~\ref{ThmKendallTransform}.


\subsection{Proof of Theorem~\ref{ThmPoly}}
\label{SecProofPoly}

We use an approach similar to the proof of
Theorem~\ref{ThmKendallTransform}.

\begin{lemma}
\label{LemPolyStepI}
The kernels $\polyKer{\power}$, $\NorPolyKer{\power}$, and
$\LamPolyKer{\power}{\nu}$ are linear combinations of the functions defined by the representation  $\tau_{(\max\{d -2\power, 1\},
  1, \ldots, 1)}$.
\end{lemma}
\begin{proof}
We express the function $\sigma \mapsto \polyKer{\power}(\sigma)$ as a linear combination of the functions defined by the representation $\tau_{(\max\{\numObj - 2\power, 1\}, 1, \ldots, 1)}$. The same property can be proved for $\NorPolyKer{\power}$ and $\LamPolyKer{\power}{\band}$ analogously.

We first analyze the case $2\power < \numObj$. By definition, we have
\begin{align*}
\polyKer{\power}(\sigma) = \left( 1 +
\kernel_\tau(\sigma)\right)^\power = \left(2 - \frac{2}{{\numObj
    \choose 2}}\ii(\sigma)\right)^\power & = 2^\power \sum_{r =
  1}^\power (-1)^r {\power \choose r} \ii(\sigma)^r\\ &= 2^\power
\sum_{r = 1}^\power (-1)^r {\power \choose r} \left(\sum_{i < j}
\indi_{\{\sigma(i) > \sigma(j)\}}\right)^r,
\end{align*}
showing that the polynomial kernel $\polyKer{\power}$ is a linear
combination of products of functions $\indi_{\{\sigma(i) >
  \sigma(j)\}}$. The products of these functions contain at most
$\power$ terms, which means there are at most $2\power$ values
$\sigma(i_1)$, $\sigma(i_2)$, ..., $\sigma(i_{2\power})$ on which the
product function depends. But the indicator functions for events of
the form $\{ \sigma(i_1) = j_1, \ldots \sigma(i_{2\power}) =
j_{2\power} \}$ form a basis for all the functions that depend only on the
values $\sigma(i_1)$, $\sigma(i_2)$, $\ldots$ $\sigma(i_\power)$. The
conclusion follows for the case $2\power < \numObj$ because these
indicator functions are exactly the functions defined by the
representation $\tau_{(\max\{\numObj - 2\power, 1\}, 1, \ldots, 1)}$.

The case $2\power \geq \numObj$ follows analogously once we observe
that any product of $2\power$ indicator functions $\indi_{\{\sigma(i)
  > \sigma(j)\}}$ is determined by $\numObj - 1$ values $\{
\sigma(i_1), \ldots, \sigma(i_{\numObj - 1}) \}$.  (To be clear, this is because given $\numObj - 1$ such
values, the $\numObj^{th}$ value is fixed).
\end{proof}

Then, by the James submodule theorem together with the linear
independence of the functions defined by the irreps $\rho_\lambda$, we
find that the Fourier transforms of the three polynomial kernels are zero at
all irreps $\rho_\lambda$ with $\lambda \lhd (\max\{\numObj -
2\power, 1\}, 1, \ldots, 1)$. The first part of Theorem~\ref{ThmPoly} is now proved. 

To prove the second part of Theorem~\ref{ThmPoly} we make use of
feature maps of the three polynomial kernels. Up
to constants, the feature maps for the three kernels
$\polyKer{\power}$, $\NorPolyKer{\power}$, and
$\LamPolyKer{\power}{\nu}$ are the same. For simplicity, we work with
the kernel $\polyKer{\power}$. All the arguments presented here extend
to the other two polynomial kernels as well.

We now give a recursive construction of the feature maps
\mbox{$\EmbdPoly{\power} \colon\simg \rightarrow \RR^{\left(1 +
    {\numObj \choose 2}\right)^\power}$} that satisfy the relation
\mbox{$\polyKer{\power}(\sigma, \sigma') =
  \EmbdPoly{\power}(\sigma)^\top \EmbdPoly{\power}(\sigma')$.}  First,
we use the feature map of the Kendall kernel to construct
$\EmbdPoly{1}$; in particular, the map \mbox{$\EmbdPoly{1} \colon
  \simg \rightarrow \RR^{1 + {\numObj \choose 2}}$} is defined by
\begin{align*}
\EmbdPoly{1}(\sigma)_{t_0} \defn 1 \quad \text{ and } \quad
\EmbdPoly{1}(\sigma)_{t_r} \defn \sqrt{{\numObj \choose
    2}^{-1}}\left(2\indi_{\{\sigma(i_r) < \sigma(j_r)\}} -1 \right),
\end{align*}
where the coordinates are indexed by the unordered pair $t_{0} = \{-1,
0\}$ and the ${\numObj \choose 2}$ unordered pairs $t_r = \{i_r,
j_r\}$ with $i_r, j_r \in [d]$ and $i_r < j_r$. We denote the set of
these unordered pairs by
\begin{align}
\label{EqTunordered}
\twosets \coloneqq \left\{t_0, t_1, \ldots, t_{{\numObj \choose
    2}}\right\}.
\end{align}
The feature map $\EmbdPoly{1}$ clearly satisfies $\polyKer{1} (\sigma,
\sigma') = \Phi_1(\sigma)^\top \Phi_1(\sigma')$. Now we use the map
$\EmbdPoly{\power -1}$ to construct a feature map $\EmbdPoly{\power}$
for $\power \geq 1$. By definition,  we have
\begin{align*}
\polyKer{\power}(\sigma, \sigma') & = \left(1 + \kendall(\sigma,
\sigma')\right)\left(1 + \kendall(\sigma, \sigma')\right)^{\power - 1}
= \EmbdPoly{1}(\sigma)^\top \EmbdPoly{1}(\sigma')
\EmbdPoly{\power - 1}(\sigma')^\top \EmbdPoly{\power - 1}(\sigma) \\ 
& = \tr\left(\EmbdPoly{1}(\sigma)^\top \EmbdPoly{1} (\sigma')
\EmbdPoly{\power - 1}(\sigma')^\top \EmbdPoly{\power -
  1}(\sigma)\right)\\ 
& = \tr \left(\left(\EmbdPoly{1}(\sigma) \EmbdPoly{\power -
  1}(\sigma)^\top\right)^\top \EmbdPoly{1}(\sigma')\EmbdPoly{\power -
  1}(\sigma')^\top\right).
\end{align*}
Therefore, the polynomial kernel of degree $\power$ between $\sigma$
and $\sigma'$ is equal to the inner product of the matrices
$\Phi_{1}(\sigma)\Phi_{\power -1}(\sigma)^\top$ and
$\Phi_{1}(\sigma')\Phi_{\power -1}(\sigma')^\top$. By induction, we
see that $\EmbdPoly{\power}$ can be obtained from $\EmbdPoly{1}$ by
taking the outer product with itself $\power$ times, meaning that the
embedding $\Phi_\power: \simg \rightarrow \RR^{\left(1 + {\numObj
    \choose 2}\right)^\power}$ can be expressed as
\label{EqPolyMap}
\begin{align}
\EmbdPoly{\power}(\sigma)_{s_1 s_2 \ldots s_\power} = \prod_{i =
  1}^\power \EmbdPoly{1}(\sigma)_{s_i},
\end{align}
where $s_1, s_2, \ldots, s_\power$ is a sequence of elements of $\twosets$.

The following lemma is the key result that allows us to show that the
three polynomial kernels of degree greater or equal than $\numObj - 1$
are characteristic.

\begin{lemma}
\label{LemLinInd}
The vectors $\{\EmbdPoly{\numObj - 1}(\sigma) \, \mid \, \sigma \in
\simg \}$ are linearly independent.
\end{lemma}

Since it is more involved, we deffer this proof to Section~\ref{AppPolyStepII}; here we provide some
intuition for the argument.  By construction, each entry of
$\Phi_{\numObj - 1}$ is equal to a product of up to $\numObj - 1$
terms $2\indi_{\{\sigma(i) < \sigma(j)\}} - 1$ times a constant. The
key property that makes the result true is that the indicator
functions $\indi_{\{\sigma = \sigma_r\}}$ can be expressed as a product
of $\numObj - 1$ indicator functions $\indi_{\{\sigma(i) <
  \sigma(j)\}}$. For example, when $d = 3$, the product
$\indi_{\{\sigma(1) < \sigma(3)\}} \indi_{\{\sigma(3) < \sigma(2)\}}$
is equal to the indicator function of the permutation
$[1,3,2]$. Moreover, the degree $\numObj - 1$ is the smallest with
this property. 

As mentioned previously, a universal kernel on the symmetric group is
also characteristic. Hence, it suffices to show that the polynomial
kernel $\polyKer{\numObj - 1}$ is universal.  Therefore, it is enough
to check that the Gram matrix $M_\tau = [\polyKer{\numObj -
    1}(\sigma_i , \sigma_j)]$ is invertible, where $\sigma_1$,
$\sigma_2$, \ldots, $\sigma_{\numObj!}$ enumerate all the elements of
$\simg$. The Gram matrix can be written as
\begin{align*}
M_\tau = \begin{bmatrix} \EmbdPoly{\numObj -
    1}(\sigma_1)^\top\\ \vdots\\ \EmbdPoly{\numObj -
    1}(\sigma_{\numObj!})^\top
\end{bmatrix}
\begin{bmatrix}
\EmbdPoly{\numObj - 1}(\sigma_1) & \cdots & \EmbdPoly{\numObj -
  1}(\sigma_{\numObj!})
\end{bmatrix}
\end{align*}
because $\polyKer{\numObj - 1}(\sigma_i, \sigma_j) = \EmbdPoly{\numObj
  - 1}^\top(\sigma_i) \EmbdPoly{\numObj - 1}(\sigma_j)$. From
Lemma~\ref{LemLinInd}, we know that the vectors $\EmbdPoly{\numObj -
  1}$ are independent, and hence the Gram matrix $M_\tau$ is full
rank, which completes the proof of Theorem~\ref{ThmPoly}.


\subsection{Proof of Theorem~\ref{ThmMallows}}
\label{SecProofMallows}

By Bochner's theorem and the Fourier inversion theorem it suffices
to show that the Mallows kernel is characteristic or universal. 

We first give a direct proof that the Mallows kernel is 
universal. 
Theorem~\ref{ThmPoly} shows that the $\nu$-normalized
polynomial kernel $\LamPolyKer{\power}{\band}$ defined by
\begin{align*}
\LamPolyKer{\power}{\band}(\sigma, \sigma') =
e^{-\frac{\band}{2}{\numObj \choose 2}} \left(1 + \band{\numObj
  \choose 2}\frac{\kendall(\sigma,
  \sigma')}{2\power}\right)^\power
\end{align*}
is characteristic and universal when the degree $\power$ is greater or equal than
$\numObj - 1$. Moreover, we saw that as the degree $\power$ increases
to infinity, the kernel $\LamPolyKer{\power}{\band}$ converges to the
Mallows kernel $\mallows$. Therefore, it is not surprising that
the Mallows kernel is universal since it is the limit of
universal kernels.

Let us now make this rough argument precise.  We need to show that the 
Gram matrix $M_m = [\mallows(\sigma_i,
  \sigma_j)]$ is strictly positive definite; here the permutations
$\sigma_1$, $\sigma_2$, \ldots, $\sigma_{\numObj!}$ enumerate the
elements of $\simg$.

Recall that the Hadamard product between two matrices $A$ and $B$ of
the same dimensions, denoted by $A\circ B$, is formed by taking
elementwise-product of the entries; we use $A^{\circ \power}$ to
denote the Hadamard product of the matrix $A$ with itself $\power$
times. By Schur's theorem, the Hadamard product $A\circ B$ of any two
PSD matrices is also PSD. Let $M_\tau = \frac{\band}{2} {\numObj
  \choose 2}[\kendall(\sigma_i, \sigma_j)]$. Performing a Taylor
series expansion of the exponential function yields
\begin{align*}
e^{\frac{\nu}{2}{\numObj \choose 2}}M_m &= e^{\frac{\nu}{2}{\numObj
    \choose 2}}[\mallows(\sigma_i, \sigma_j)] =
e^{\frac{\nu}{2}{\numObj \choose 2}}[e^{-\band n_d(\sigma_i,
    \sigma_j)}] = \sum_{i = 0}^\infty \frac{1}{i!}M_\tau^{\circ i},
\end{align*}
where the series on the right hand side is entry-wise absolutely
convergent. For some $0\leq \alpha_i \leq 1$, re-arranging terms yields
\begin{align}
e^{\frac{\nu}{2}{\numObj \choose 2}}M_m &= \sum_{i = 0}^{\numObj -
  1}{\numObj - 1 \choose i}\frac{1}{(\numObj - 1)^i}M_\tau^{\circ i} +
\sum_{i = 0}^\infty \alpha_i \frac{1}{i!} M_\tau^{\circ i} \nonumber
\\
\label{EqMallowsExpansion}
& = \left(1 + \frac{M_\tau}{\numObj - 1}\right)^{\circ (\numObj - 1)}
+ \sum_{i = 0}^\infty \alpha_i \frac{1}{i!} M_\tau^{\circ i}.
\end{align}
The first term in the right hand side of~\eqref{EqMallowsExpansion} is
the Gram matrix of the $\nu$-normalized polynomial kernel of degree
$\numObj - 1$, and thus it is a strictly positive definite matrix. The
second term is a positive semi-definite matrix because of Schur's
theorem. Hence $M_m$ is strictly positive definite and
Theorem~\ref{ThmMallows} is now proved.

For completeness, we show that the Mallows kernel is characteristic 
in two other ways. First of all, because of the feature embedding
 of the Kendall kernel, it can be viewed as the standard 
Gaussian kernel on $\RR^{\numObj \choose 2}$ restricted to $2^{\numObj \choose 2}$.
Then, since the Gaussian kernel is characteristic, the Mallows
kernel has to be characteristic. 

As yet another proof, we note that the result of
Theorem~\ref{ThmMallows} can be obtained via a more abstract argument,
using the results of~\cite{christmann2010universal}.  Given a compact
metric space $X$ and a separable Hilbert space $\mathcal{H}$, let
$\Psi\colon X \rightarrow \mathcal{H}$ a continuous and injective
map. The authors show that the kernel $\kernel$ on $X \times X$
given by
\begin{align}
\label{EqChrisKer}
\kernel(x,y) = e^{-\nu \|\Psi(x) - \Psi(y)\|_\mathcal{H}^2}
\end{align}
is universal. The symmetric group is a compact metric space and we can choose
$\Psi = \Phi$, the feature map of the Kendall kernel. We can thus
conclude that the kernel defined in equation~\eqref{EqChrisKer} is
universal and characteristic; since it equals Mallows' kernel
up to constants, the claim of Theorem~\ref{ThmMallows} follows.



\section{Background in Representation Theory}
\label{SecAppRep}

In this section, we present further notions and results about the
representation theory for the symmetric group.  Our exposition is
brief and covers only the essential results needed in our work. For a
more detailed introduction good resources include the thesis
of~\cite{kondor2008group} and the appendices
by~\cite{huang2009fourier}, with a concise summary also given
by~\cite{kondor2010ranking}.  More detailed presentations can be found
in~\cite{diaconis1988group},~\cite{sagan2013symmetric},
or~\cite{fulton1991representation}, ordered according to increasing
levels of abstraction.


\subsubsection*{Groups}

A group $(G, \cdot)$ is a set $G$ endowed with a multiplicative
operation $\cdot\colon G\times G \rightarrow G$ such that
\begin{enumerate}[leftmargin=*]
\item[(a)] there exists an element $e\in G$ called the identity
  element such that $e\cdot g = g \cdot e = g$ for all $g\in G$.
\item[(b)] $g_1\cdot (g_2 \cdot g_3) = (g_1 \cdot g_2) \cdot g_3$ for
  all $g_1, g_2, g_3 \in G$.
\item[(c)] for any element $g\in G$, there exists $g^{-1}\in G$ such
  that $g\cdot g^{-1} = g^{-1}\cdot g = e$.
\end{enumerate}

It is easy to check that $(\RR, +)$ or $(\RR, \cdot)$ are examples of
groups. It is also straightforward to check that the set of
permutations together with the operation of composition form a group,
called the symmetric group.  Notice that we do not require $g_1\cdot
g_2 = g_2 \cdot g_1$. A group with this property is called
\emph{commutative} or \emph{abelian}. Abelian groups are easier to
study than non-abelian ones. Unfortunately, the symmetric group is not
abelian.


\subsubsection*{Equivalent Representations}

Two representations $\rho_1$ and $\rho_2$ are \textbf{equivalent} if
they have the same dimension and if there exits an invertible matrix
$C$ such that $\rho_1(\sigma) = C^{-1} \rho_2(\sigma) C$ for all
$\sigma \in \simg$. In other words, two representations are equivalent
if there exists a change of basis that makes one of them equal to the
other.  We use $\rho_1 \equiv \rho_2$ to denote the equivalence of the
representations $\rho_1$ and $\rho_2$.

For any representation $\rho_1$, there exists an equivalent
representation $\rho_2$ such that each matrix $\rho_2(\sigma)$ is
unitary (i.e. $\rho_2(\sigma)^* \coloneqq
\overline{\rho_2(\sigma)}^\top = \rho_2(\sigma)^{-1} =
\rho_2(\sigma^{-1})$). Therefore, we can always assume that the
representations we are working with are unitary.

Furthermore, in the case of the irreps of the symmetric group, there
exist bases such that each representation $\rho_\lambda$ is real, and
hence orthogonal. The irreps in these bases are known as Young's
orthogonal representations, and throughout this paper we work with
these forms of $\rho_\lambda$.


\subsubsection*{Irreps}

We already said that an irreducible representation is a representation
that is not equivalent to a direct sum of representations.  The
symmetric group, in fact any finite group, has a finite number of
pairwise inequivalent irreps. Let us consider a maximal set of
pairwise inequivalent irreps. There can be multiple such sets, but
they are the same up to equivalence.  To be more precise, between two
maximal sets of irreps there exists a bijection such that an irrep in
the first set is mapped to an equivalent irrep in the other set.

A fundamental result in representation theory states than any
representation is equivalent to a direct sum of irreps. That is, each
representation $\rho$ can be decomposed into the direct of sum of some
irreducible representations $\rho_1$, $\rho_2$, \ldots, $\rho_{k}$
with some multiplicities $m_1$, $m_2$, \ldots, $m_k$:
\begin{align*}
\rho \equiv \bigoplus_{i = 1}^k \bigoplus_{j = 1}^{m_i} \rho_i.
\end{align*}
Let us recall that the entries of each representation $\rho \colon
\simg \rightarrow \CC^{d_\rho \times d_\rho}$ define $d_\rho^2$
functions $\sigma \mapsto \rho(\sigma)_{ij}$ on the symmetric
group. The functions defined by Young's orthogonal representations
form a basis for the space of functions $f\colon \simg \rightarrow
\CC$. This result is important and this work exploits it extensively.


\subsubsection*{The Fourier Transform}

We saw that the \textbf{Fourier transform} of a function $f \colon
\simg \rightarrow \CC$ is a map from representations to matrices, and
it is given by
\begin{align*}
\widehat f(\rho) = \sum_{\sigma \in \simg} f(\sigma) \rho(\sigma),
\end{align*}
where $\rho$ is a representation of the symmetric group.

This Fourier transform has properties similar to those of its
counterpart over the real numbers. First of all, there exists a
\textbf{Fourier inversion formula} and it takes the form
\begin{align*}
f(\sigma) = \frac{1}{\numObj !}\sum_{\lambda \vdash \numObj} d_\lambda
\tr\left(\rho_\lambda(\sigma^{-1})\widehat f(\rho_\lambda)\right).
\end{align*}
The Fourier transform on the symmetric also satisfies the
\textbf{Plancherel formula}:
\begin{align*}
 \sum_{\sigma \in \simg} f(\sigma^{-1})g(\sigma) =
 \frac{1}{\numObj!}\sum_{\lambda \vdash \numObj} d_\lambda
 \tr\left(\hat f(\rho_\lambda) \hat g(\rho_\lambda)\right).
\end{align*}

A third familiar property is that the Fourier transform of the
convolution of two functions is the product of the Fourier transforms
of the individual functions.  The \textbf{convolution} of two
functions $f,g:\simg \rightarrow \CC$ is defined by $f\ast g(\pi) =
\sum_{\sigma\in \simg} f(\pi \sigma^{-1})g(\sigma)$.


\subsubsection*{Ferrer diagrams, Young tableaux, and Young tabloids}

As mentioned in Section~\ref{SecKendallMallowsFreq}, it is natural to index the
irreps of $\simg$ by partitions $\lambda$ of $d$. The exact
correspondence is not easy to describe, but it is useful to understand
how to visualize the partitions $\lambda$ and the corresponding irrep
$\rho_\lambda$.

The partitions $\lambda \vdash d$ are represented graphically in the
form of \textbf{Ferrer's diagrams}. The diagram of a partition
$\lambda = (\lambda_1, \ldots, \lambda_r)$ is formed by boxes placed
in rows such that row $i$ contains $\lambda_i$ boxes. For example, the
partitions of $4$ are $(4)$, $(3,1)$, $(2,2)$, $(2,1,1)$, and
$(1,1,1,1)$, represented as:
\begin{align*}
\begin{ytableau}
\  &\  & \ & \ 
\end{ytableau}
\quad \quad 
\begin{ytableau}
\  &\  & \ \\
\ 
\end{ytableau}
\quad \quad 
\begin{ytableau}
\ &\ \\ \ &\
\end{ytableau}
\quad \quad
\begin{ytableau}
\ &\ \\ \ \\ \
\end{ytableau}
\quad \quad 
\begin{ytableau}
\ \\ \ \\ \ \\ \
\end{ytableau}
\end{align*}

In this graphical representation, a wider partition is higher in the
partial ordering, while a taller partition is lower in the partial
ordering.

A Ferrer diagram with the elements of the set $\{1,2,\ldots, d\}$ in
its boxes is called a \textbf{Young tableau}. Young tableaux in which
the rows are viewed as sets are called \textbf{Young tabloids}. To
emphasize that the rows of a Young tabloid are not ordered we drop the
vertical lines in the graphical representation. For example, the Young
tabloids of the partition $(2,1)$ are \ytableausetup{centertableaux,
  tabloids}
\begin{align*}
\begin{ytableau}
1 & 2 \\ 3
\end{ytableau}
\quad \quad
\begin{ytableau}
1 & 3 \\ 2
\end{ytableau}
\quad \quad 
\begin{ytableau}
2 & 3 \\ 1
\end{ytableau}
\end{align*}
In what follows, we adopt the shorthand notation $\sigma(\{1,3\})
\defn \{\sigma(1), \sigma(3)\}$. When we are interested in the subset
of permutations $\sigma \sim P$ that satisfy $\sigma(\{1,3\}) =
\{2,5\}$, $\sigma(\{2,4\}) = \{1,4\}$ and $\sigma(\{5\}) = \{3\}$, we
express this as the permutations that satisfy
\begin{align}
\sigma \left(
\begin{ytableau}
1 & 3 \\ 2 & 4\\ 5
\end{ytableau} 
\right) = 
\begin{ytableau}
2 & 5 \\ 1 & 4\\ 3
\end{ytableau}
\label{eq:tableaux_prob}
\end{align}


\subsubsection*{Overcomplete Representations and James' Submodule Theorem}

In studying irreps or the Fourier transforms of functions it is often
useful to consider reducible representations that have an easy to
understand interpretation and contain copies of the irreps. We have
seen in Section~\ref{SecProofKendall} that the representations
$\tau_\lambda$ play such a role. We now define these representations
for a general partition $\lambda$.

Let $\{t_1\}$, $\{t_2\}$, ..., $\{t_l\}$ be an enumeration of all Young tabloids\footnote{It is standard to use $\{t\}$ to denote a Young tabloids and $t$ to denote a Young tableaux because the former are equivalence classes of the latter.} of some partition $\lambda \vdash d$. The representation $\tau_\lambda$ takes values in $\RR^{l\times l}$ and is defined by
\begin{equation}
\label{EqOverCompDefn}
[\tau_\lambda(\sigma)]_{ij} = \left\{
\begin{array}{cc}
1 & \text{if } \sigma(\{t_i\}) = \{t_j\}\\
0&\text{otherwise}
\end{array}
\right.
\end{equation}

We note that the Fourier transform of a probability measure $\probP$ at the representation $\tau_\lambda$ encodes marginal probabilities:
\begin{align*}
[\widehat \probP(\tau_\lambda)]_{ij} = \sum_{\sigma \in \simg} \probP(\sigma)[\tau_\lambda(\sigma)]_{ij} = P(\sigma(\{t_i\}) = \{t_j\}).
\end{align*}

Therefore the Fourier transform at this representation has a concrete interpretation in ``time domain''. Nonetheless, because of the Fourier inversion formula we want to understand the properties of the kernel functions at irreps. James' Submodule Theorem give a decomposition of $\tau_\lambda$ into irreps. We state the form of the theorem presented by~\cite{huang2009fourier}.

\begin{theorem*}{\normalfont \textbf{[James' Submodule Theorem]}}
There exist orthogonal matrices $C_\lambda$ and integers $K_{\lambda \mu} \geq 0$ so that
\begin{equation}
    C_\lambda^\top \tau_\lambda(\sigma) C_\lambda = \bigoplus_{\mu \unrhd \lambda} \bigoplus_{l = 1}^{K_{\lambda \mu}}\rho_\mu(\sigma) \text{, for all } \sigma \in \simg. 
\end{equation}
Furthermore, $K_{\lambda\lambda} = 1$ for all $\lambda \vdash d$. 
\end{theorem*}

The integers $K_{\lambda, \mu}$ are known as Kostka's numbers and there are methods to compute them. For example, we have already mentioned in Section~\ref{SecProofKendall} that 
\begin{align}
\label{EqTauI}
\tau_{(n)} &\equiv \rho_{(n)}\\
\label{EqTauII}
\tau_{(n-1,1)} &\equiv \rho_{(n)} \oplus \rho_{(n-1,1)}\\
\label{EqTauIII}
\tau_{(n-2,2)} &\equiv \rho_{(n)} \oplus \rho_{(n-1,1)} \oplus \rho_{(n-2,2)}\\
\label{EqTauIV}
\tau_{(n-2,1,1)} &\equiv \rho_{(n)} \oplus \rho_{(n-1,1)} \oplus \rho_{(n-1,1)} \oplus \rho_{(n-2,2)} \oplus \rho_{(n-2,1,1)}.
\end{align}


\section{Miscellaneous Proofs}
\label{SecAppProofs}

In this appendix, we collect the proofs of various other results.
 
\subsection{Proof of Lemma~\ref{LemLinInd}}
\label{AppPolyStepII}

Recall from equation~\eqref{EqTunordered} that for $i_r,j_r \in [d]$
with $i_r < j_r$, we use $t_r = \{i_r,j_r\}$ to denote unordered pairs
with an additional $t_0 = \{-1,0\}$ for convenience, and $\twosets$ to
denote the set of all such ${d \choose 2} + 1$ unordered pairs. In
equation~\eqref{EqPolyMap}, the definition of the feature map
$\EmbdPoly{\numObj - 1} : \simg \to \RR^{\left({d \choose 2} +
  1\right)^{d-1}}$ implies that
\begin{align*}
 \forall s_1, s_2, \ldots s_{d-1} \in \twosets, ~ ~ \EmbdPoly{\numObj
   - 1}(\sigma)_{s_1 s_2 \ldots s_{\numObj-1}} = C_{s_1 s_2 \ldots
   s_{\numObj -1}} \prod_{s_r \neq t_0} (2\indi_{\{\sigma(i_r) <
   \sigma(j_r)\}} - 1)
\end{align*}
where the product is only over $s_r \neq t_0$ since
$\EmbdPoly{1}(\sigma)_{t_0} = 1$, and $C_{s_1 s_2 \ldots s_{\numObj
    -1}} = {d \choose 2}^{|\{r: s_r \neq t_0 \}|/2}$ is a positive
constant independent of $\sigma$. We use the convention that an empty
product evaluates to 1.

Now define a new feature map $\EmbdNor{\numObj - 1}~:~
\simg~\to~\RR^{\left({d \choose 2} + 1\right)^{d-1}}$ as
 \begin{align*}
 \forall \sigma \in \simg, ~ \forall s_1, s_2, \ldots s_{d-1} \in
 \twosets~, ~ ~ \EmbdNor{\numObj - 1}(\sigma)_{s_1 s_2 \ldots
   s_{\numObj-1}} = \prod_{s_r \neq t_0} (2\indi_{\{\sigma(i_r) <
   \sigma(j_r)\}} - 1)
 \end{align*}
Let $C$ represent an invertible diagonal matrix of the constants
$C_{s_1 s_2 \ldots s_{\numObj-1}}$.  Then note that
\begin{align*}
\forall \sigma \in \simg~, ~ \EmbdPoly{\numObj-1}(\sigma) = C
\EmbdNor{\numObj-1}(\sigma).
\end{align*}
Consequently, the vectors $\{\EmbdPoly{\numObj - 1}(\sigma)\}_{\sigma
  \in \simg}$ are linearly independent if and only if
$\{\EmbdNor{\numObj - 1}(\sigma)\}_{\sigma \in \simg}$ are linearly
independent. We work with $\{\EmbdNor{\numObj - 1}(\sigma)\}_{\sigma
  \in \simg}$ because its entries are always $\pm1$.

\emph{\paragraph{Claim.} If $\{\alpha(\sigma)\}_{\sigma \in \simg}$
  are $d!$ real coefficients such that
\begin{align}
\label{EqIndExpansion}
\sum_{\sigma \in \simg}\alpha(\sigma) \overline\Phi_{\numObj -
  1}(\sigma) = \mathbf{0} ~\in~ \RR^{\left({d \choose 2} +
  1\right)^{d-1}},
\end{align}
then each coefficient $\alpha(\sigma)$ is equal to
zero. \\
}

In what follows, we drop repeated occurrences of $t_0$ when indexing
the coordinates of $\EmbdNor{\numObj - 1}$ without risking
confusion. For example, $\EmbdNor{\numObj - 1}(\sigma)_{t_2}$ means
$\EmbdNor{\numObj - 1}(\sigma)_{t_2 t_0 \ldots t_0}$, where $t_0$ is
repeated $\numObj - 2$ times. Now observe that $\EmbdNor{\numObj -
  1}(\sigma)_{t_0} = 1$ for all $\sigma$, implying that 
\begin{subequations}
\begin{align}
\label{EqnGobiOne}
\sum_\sigma
\alpha (\sigma) &= 0.
\end{align}
By construction, we have $\EmbdNor{\numObj - 1}(\sigma)_{t_r} =
2\indi_{\{\sigma(i_r) < \sigma(j_r)\}} - 1$, and hence
\begin{align}
\label{EqnGobiTwo}
\sum_{\{\sigma(i) < \sigma(j)\}} \alpha(\sigma) - \sum_{\{\sigma(i) >
  \sigma(j)\}} \alpha(\sigma) = 0 \qquad \mbox{for all sets $s = \{i,
  j\}\in \twosets$.}
\end{align}
\end{subequations}
Equations~\eqref{EqnGobiOne} and~\eqref{EqnGobiTwo} imply that
\begin{align}
\label{EqBaseCase}
\sum_{\{\sigma(i) < \sigma(j)\}} \alpha(\sigma) = 0 \;\text{ and }\;
\sum_{\{\sigma(i) > \sigma(j)\}} \alpha(\sigma) = 0.
\end{align}
From now on, for a given unordered pair $s = \{i, j\}\in \twosets$, we
introduce the shorthand notation $+s \defn \{\sigma \colon \sigma(i) <
\sigma(j)\}$, with $-s$ denoting its complement.  Moreover, we write
several such signed unordered pairs next to each other we mean the
intersection of the two sets. For example $+s_1-s_2$ means $s_1 \cap
s_2^c$.

We use induction to show that each $\alpha(\sigma)$ is zero.  Assume
that for some fixed integer $\power$ and for all choices of $\power$
unordered pairs $s_{1}, \ldots, s_{\power} \in \twosets$ and for all
possible binary signs $\epsilon_{1}, \ldots, \epsilon_{\power} \in \{
+1, -1\}$ the following holds:
\begin{align*}
\sum_{\epsilon_{1}s_{1} \ldots
  \epsilon_{\power}s_{\power}}\alpha(\sigma) = 0.
\end{align*}
We show that this property holds for all choices of $\power + 1$
unordered pairs and binary signs. The base case $p = 1$ has been shown in Equation~\eqref{EqBaseCase}.

Fix a sequence of $\power + 1$ distinct pairs $s_{1}$, $s_{2}$,
\ldots, and $s_{\power}$, all distinct from $t_0$. Then, each sequence
$\epsilon_1s_{1}$, $\epsilon_2s_{2}$, \ldots, $\epsilon_\power
s_{\power}$ can be encoded with a vector in $\{-1, +1\}^{\power +
  1}$. For a given sign vector $\epsilon$ in $\{-1, +1\}^{\power + 1}$
let $\sign(\epsilon)$ be equal to the product of the entries of
$\epsilon$. Therefore, $\sign(\epsilon)$ is $+1$ if the vector
$\epsilon$ contains an even number of $-1$ entries, and is $-1$
otherwise. Then, we have
\begin{align}
\label{EqInduction}
\sum_{\sigma \in \simg} \alpha(\sigma) \overline\Phi_{\numObj -
  1}(\sigma_i)_{s_1\ldots s_{\power + 1}} = 0 ~\Longrightarrow
\sum_{\epsilon \in \{-1, + 1\}^{\power + 1}}
\sign(\epsilon)\sum_{\epsilon_1 s_{1}\ldots \epsilon_{\power +
    1}s_{\power + 1}} \alpha(\sigma) = 0.
\end{align}

The signed pairs $-s_{1}$ and $+s_{1}$ are complements of each
other. Therefore, we have
\begin{align*}
\sum_{-s_{1} \epsilon_2 s_{2}\ldots \epsilon_{\power + 1}s_{\power +
    1}} \alpha(\sigma) + \sum_{+s_{1} \epsilon_2 s_{2} \ldots
  \epsilon_{\power + 1}s_{\power + 1}} \alpha(\sigma) =
\sum_{\epsilon_2 s_{2}\ldots \epsilon_{\power + 1}s_{\power + 1}}
\alpha(\sigma).
\end{align*}
This property holds for all pairs $s_{j}$, not just for
$s_{1}$. Furthermore, by the induction step we know that the right
hand side of the above equation equals zero. More generally, if $\epsilon$ and $\xi$
are two sign vectors in $\{ - 1, + 1\}^{\power + 1}$ that differ only
in a coordinate, we have
\begin{align}
\label{EqDistanceOne}
\underbrace{\sum_{\epsilon_1s_{1} \epsilon_2 s_{2}\ldots
    \epsilon_{\power + 1}s_{\power + 1}} \alpha(\sigma)}_{f(\epsilon)}
+ \underbrace{\sum_{\xi_1s_{1} \xi_2 s_{2}\ldots \xi_{\power +
      1}s_{\power + 1}} \alpha(\sigma)}_{f(\xi)} = 0.
\end{align}
Let $G$ be the standard graph on the hypercube $\{ - 1,
+ 1\}^{\power + 1}$, i.e. the graph with node set equal to $\{ - 1, +
1\}^{\power + 1}$ that connects to nodes by an edge only if they
differ in a single coordinate. Then, observation~\eqref{EqDistanceOne} 
immediately implies that for any two sign vectors  $\epsilon$ and $\xi$ at distance
two of each other in the graph $G$, we have
$f(\epsilon) = f(\xi)$.
In fact, it is straightforward that any pair of nodes
$\epsilon$ and $\xi$ that are at an even distance apart satisfy
$f(\epsilon) = f(\xi)$. 

It is easily checked that two nodes
$\epsilon$ and $\xi$ are at an even distance away only if
$\sign(\epsilon) = \sign(\xi)$. Therefore, if $\sign(\epsilon) =
\sign(\xi)$, then $f(\epsilon) = f(\xi)$. Moreover,
equation~\eqref{EqInduction} implies that
\begin{align*}
\sum_{\epsilon \colon \sign(\epsilon) = 1} f(\epsilon) -
\sum_{\epsilon \colon \sign(\epsilon) = -1} f(\epsilon) = 0.
\end{align*}
We also know that $\sum_{\epsilon\in \{-1,+1\}^{\power + 1}}
f(\epsilon) = 0$ because $\sum_{\sigma \in \simg} \alpha(\sigma) =
0$. Then, $\sum_{\sign(\epsilon) = 1} f(\epsilon)~=~0$ and
$\sum_{\sign(\epsilon) = -1} f(\epsilon)~=~0$. But the terms inside
each of these sums are equal to each other, hence $f(\epsilon) = 0$
for all $\epsilon\in \{-1, + 1\}^{\power + 1}$. This completes the
induction step.

Finally, because for any permutation $\sigma$ there
exists a sequence of $\numObj - 1$ unordered pairs $\indi_{\{\sigma(i) <
  \sigma(j)\}}$ that uniquely determine it, for each permutation
$\sigma$ we can choose sets $\epsilon_1 s_{1}$, \ldots,
$\epsilon_{\numObj - 1} s_{\numObj - 1}$ such that $\sigma$ is the
only permutation that is contained in all of them. Then, by what have
proven so far, we find $\alpha(\sigma) = 0$ for all $\sigma \in \simg$
and the conclusion follows.

\subsection{Proving that $n_\numObj (\sigma, \sigma')$ is right invariant}
\label{AppProofPropRightInv}

We need to check that 
$n_d(\sigma, \sigma') = n_d( \sigma\circ \pi,  \sigma'\circ \pi)$
for all $\pi \in \simg$.  By definition, we have
\begin{align*}
&\sum_{i< j}\left[\indi_{\{\sigma(i)<
      \sigma(j)\}}\indi_{\{\sigma'(i)>\sigma'(j)\}} +
    \indi_{\{\sigma(i)>
      \sigma(j)\}}\indi_{\{\sigma'(i)<\sigma'(j)\}}\right] =
  \\ &=\sum_{i< j}\left[\indi_{\sigma(\pi(i))<
      \sigma(\pi(j))\}}\indi_{\{\sigma'(\pi(i)))>\sigma'(\pi(j))\}} +
    \indi_{\{\sigma(\pi(i))>
      \sigma(\pi(j))\}}\indi_{\{\sigma'(\pi(i))<\sigma'(\pi(j))\}}\right]
\end{align*}

Note that the permutation $\pi$ just maps the sets $\{i,j\}$
bijectively to the sets $\{\nu(i),\nu(j)\}$. Since we are summing over
all the pairs, it means that the two sums must be equal. By choosing
$\pi = \sigma^{-1}$ we get that $ n_d(\sigma, \sigma') = n_d( e,
\sigma'\circ \sigma^{-1}), $ where $e$ is the identity permutation. By
definition $n_d( e, \sigma'\circ \sigma^{-1}) = \ii(\sigma'\circ
\sigma^{-1})$.

\subsection{$\mmd_\kernel$ in Fourier Domain}
\label{AppMMDFourier}

We show that for any kernel $\kernel$ on $\simg$ the maximum mean discrepancy can satisfies the identity:
\begin{align}
\label{EqMMDFourier}
\mmd_\kernel^2(P, Q) & = \frac{1}{d!} \sum_{\lambda \vdash \numObj}  d_\lambda \tr\left[\left(\widehat \probP(\rho_\lambda) -
  \widehat \probQ(\rho_\lambda)\right)^\top \widehat
  \kernel(\rho_\lambda) \left(\widehat \probP(\rho_\lambda) - \widehat
  \probQ(\rho_\lambda)\right) \right]
\end{align}

Let $\alpha_1$, $\alpha_2$ be two independent random permutations
sampled according to the probability distribution $\probP$. Similarly
$\beta_1$ and $\beta_2$ are independent and sampled according to
$\probQ$.  The Fourier inversion formula ensures that
\begin{align}
  \kernel(\alpha_1,\alpha_2) = \kernel(\alpha_1\alpha_2^{-1})=
  \frac{1}{d!}\sum_{\lambda\vdash d} d_\lambda \tr\left[\widehat
    \kernel(\lambda) \rho_\lambda(\alpha_2 \alpha_1^{-1})\right] =
  \frac{1}{d!} \sum_{\lambda\vdash d} d_\lambda
  \tr\left[\rho_\lambda(\alpha_1)^\top \widehat \kernel(\lambda)
    \rho_\lambda (\alpha_2)\right],
\end{align}
where the last equality follows because the irrep $\rho_\lambda$ is
one of Young's orthogonal representations.
 
Taking expectation with respect to $\alpha_1$ and $\alpha_2$ yields
\begin{align}
\EE \kernel(\alpha_1, \alpha_2) = \frac{1}{d!}\sum_{\lambda\vdash d}
d_\lambda\tr \left[\widehat \probP(\lambda)^\top\widehat k(\lambda)
  \widehat \probP(\lambda)\right].
\end{align}
In an analogous manner, we have
\begin{align*}
\EE \kernel(\beta_1, \beta_2) &= \frac{1}{d!}\sum_{\lambda\vdash d}
d_\lambda\tr \left[\widehat \probQ(\lambda)^\top\widehat
  \kernel(\lambda) \widehat \probQ(\lambda)\right] \text{~ and ~} \EE
\kernel(\alpha_1, \beta_1) ~= \frac{1}{d!}\sum_{\lambda\vdash d}
d_\lambda\tr \left[\widehat \probQ(\lambda)^\top\widehat
  \kernel(\lambda) \widehat \probP(\lambda)\right].
\end{align*}
Given these pieces, the conclusion follows because
\begin{align*}
\mmd_\kernel^2(\probP, \probQ) = \EE \kernel(\alpha_1, \alpha_2) +
\kernel(\beta_1, \beta_2) - \kernel(\alpha_1, \beta_2) -
\kernel(\alpha_2, \beta_1).
\end{align*}
In particular, see the paper~\cite{gretton2012kernel} for a proof of
this last identity.

We note that the Fourier expansion~\eqref{EqMMDFourier} of the $\mmd_\kernel^2$ 
shows that the kernel $\kernel$ is characteristic if and only if $\hat\kernel$ is 
strictly positive definite at all irreps.

{
\bibliographystyle{agsm}
\bibliography{permutations}
}

\end{document}